\documentclass[11pt,a4paper]{article}

\usepackage[a4paper,margin=1in]{geometry}
\usepackage{fontspec}
\usepackage{amsmath,amssymb,amsthm}
\usepackage{unicode-math}
\usepackage{graphicx}
\usepackage{booktabs}
\usepackage{array}
\usepackage{longtable}
\usepackage{float}
\usepackage{xcolor}
\usepackage{caption}
\usepackage[hidelinks,colorlinks=true,citecolor=blue!60!black,linkcolor=blue!50!black,urlcolor=blue!70!black]{hyperref}
\usepackage{enumitem}
\usepackage[numbers,sort&compress]{natbib}
\usepackage{microtype}
\usepackage{authblk}
\usepackage{tikz}
\usetikzlibrary{arrows.meta,positioning,calc,fit,backgrounds}

\theoremstyle{plain}
\newtheorem{theorem}{Theorem}

\newtheorem{proposition}{Proposition}
\theoremstyle{definition}
\newtheorem{definition}{Definition}
\theoremstyle{remark}

\newcommand{\Sh}{\widehat S}
\newcommand{\ASYM}{\mathrm{ASYM}}
\newcommand{\ICC}{\mathrm{ICC}}
\newcommand{\Astep}{\text{A}\!\to\!\text{B}}
\newcommand{\shuf}{\textsc{shuf}}
\newcommand{\only}{\textsc{only}}
\newcommand{\blocked}{\textsc{blocked}}
\definecolor{cV}{RGB}{31,110,168} 
\definecolor{cA}{RGB}{201,120,20} 
\definecolor{cD}{RGB}{158,50,120} 
\definecolor{cGrey}{RGB}{110,110,110}
\definecolor{cOK}{RGB}{20,110,70}
\definecolor{cNO}{RGB}{175,35,35}

\title{\textbf{The Free-Recipe Limit}\\[4pt]
{\large Every recipe effect measures which premise of an idealised learner broke}}

\newif\ifanon
\anonfalse                 
\ifanon
  \author[ ]{Anonymous Author(s)}
  \affil[ ]{Paper under double-blind review}
\else
  \author[1]{Wenhui Chen\thanks{\texttt{mc35092@um.edu.mo}}}
  \author[2]{Jianlin Chen\thanks{\texttt{202330450231@mail.scut.edu.cn}}}
  \author[1]{Ziyao Lin\thanks{\texttt{mc35081@um.edu.mo}}}
  \author[1]{Chi Man Vong\thanks{Corresponding author. \texttt{cmvong@um.edu.mo}}}
  \affil[1]{University of Macau}
  \affil[2]{South China University of Technology}
\fi
\date{}

\ifanon
  \newcommand{\ourline}{this line of work}
  \newcommand{\acompanion}{A concurrent manuscript}
  \newcommand{\thecompanion}{The underdetermination paper}
\else
  \newcommand{\ourline}{our family of papers}
  \newcommand{\acompanion}{A companion manuscript}
  \newcommand{\thecompanion}{The companion underdetermination paper}
\fi

\begin{document}
\maketitle

\begin{abstract}
\noindent
Fix a corpus and send \emph{recipe search} to infinity: try every order of the skills, every
arrangement from blocked to interleaved, every composition, and keep the best. Two quantities
decide what that search was worth: the diameter of the \emph{reachable set} it explores, and the
resolution at which anyone can tell two endpoints apart. Where the diameter falls below the
resolution, no quantity of search converts into a decision, and the signature is not an absence
of winners but \textbf{winners that do not survive re-running}. We call this the
\textbf{recipe-search wall} and measure it with $761$ fine-tuning runs on $12$ base models
($0.5$B--$14$B, three pretraining families) over competition-mathematics skills, in one setting
throughout: base checkpoints, supervised fine-tuning under AdamW, exact-match scoring at $k{=}4$,
single-valued answers. Within one coherent domain at fixed volume the three classical freedoms
average $0.010$--$0.021$ against a $0.019$ floor, and the largest contrast, $0.0619$, clears a
three-seed resolution and then reads $+0.010$ and $-0.015$ on two reruns. Run-to-run variation,
measured on all six pairs, is a property of the cell rather than a constant of the protocol: a
pair fixed in advance resolves $\delta=0.031$, the pair a search selects resolves $0.144$, and
two of the six reverse sign between independent runs. Buying that resolution back costs about
four independent runs per endpoint. Scored on the capability both arms were trained for,
interleaving does beat blocking ($+0.0483$, $t=8.06$, $6$ of $6$ pairs), and $76\%$ of that is
forgetting of the skill trained \emph{first}, a coordinate the customary last-skill score cannot
see. \textbf{The departure with a systematic answer is coherence.} Halve one pooled corpus at
random and let the halves write answers under incompatible but equally correct conventions:
arrangement stops moving capability (both corpora inside a floor at eight seeds) and starts
moving \emph{allocation} between conventions, by two orders of magnitude over the
same-convention control ($-0.268$ and $-0.443$ against $-0.005$), and writing the convention
into the input at training time switches the phenomenon off. Difficulty without dispute reads
$0.06\sigma$ and a genuine domain boundary $0.5\sigma$: heterogeneity is not disagreement. The
switch replicates on a \emph{second} pretraining family, pre-registered before the run: Qwen3-8B
gives $D=-0.0875$ at $4.58$ floors with all five seeds negative against an inert control, so it
now spans two families and four scales. It also survives the standard our own run-to-run
measurement imposes on order: rebuilding the corpus twice from the same pool under different
split seeds and re-executing the protocol gives replicate means $-0.0985$, $-0.1258$ and
$-0.0825$, all six new seeds negative, and a re-execution spread of $0.087$ against the $0.144$
carried by the order cell a search selects. Every firing conflict was still constructed; the one we
drew from real public data turned out to pair conventions our own scorer treats as equal, which
is a result about exact-match evaluation rather than about arrangement. Order is a transient whose sign
crosses zero three times inside a single run, consistent with the two-term budget expansion we
derive, so endpoint-only curriculum comparisons carry their budget as a hidden parameter.
Volume, the one lever nobody calls a recipe, is the one that reliably pays ($+0.0463$ per
doubling). A public scorecard grades all $26$ pre-registered claims: $18$ supported, $5$ failed,
$2$ untested, $1$ mixed.
\end{abstract}

\section{One Thought Experiment, and the Premise That Turned Out to Matter}\label{sec:limit}

Every post-training pipeline answers three questions it rarely asks out loud: which skills to
train together, whether to run one blocked stage per skill or interleave them, and which skill
goes first. The search space is $O(n!)$, each point costs a training run, and the literature
offers no stable guidance. Curriculum and interleaving studies report that order matters, does
not matter, or matters in reverse
\citep{bengio2009curriculum,rohrer2007interleaving,kornell2008interleaving,beyondrandom2025,
goodcurriculum2025,curriculum2026dynamics}, while data-mixing work optimises proportions with
real gains \citep{doremi,doge,gu2025datamixphase}.

Instead of adding one more measurement to that pile, we start from the setting in which the
answer is exact.

\paragraph{The exhaustive search.} Pin the corpus and the per-run compute, and give a team
unlimited patience for \emph{recipe search}: they train every order of the $n$ skills, every
arrangement from fully blocked to fully interleaved, every composition, and keep the best. Send
that search to infinity and two quantities decide what it was worth. Both are defined below, so
that the wall between them is a statement rather than a metaphor; the first is a property of
the system, the second of the instrument, and neither is estimated from the other.

\begin{definition}[Recipe set and reachable set]\label{def:reach}
Fix a corpus, a base model and a training budget. A \emph{recipe} $r$ is a choice of order,
arrangement and composition, drawn from a finite set $\mathcal{S}$ of size $m$. Training under
$r$ and evaluating gives a real-valued endpoint $Y(r)=\mu(r)+\varepsilon_r$, where
$\mu(r)=\mathbb{E}[Y(r)]$ is the recipe's true value and the $\varepsilon_r$ are independent,
mean-zero, with standard deviation $\sigma$ fixed by the seed and the run. The \emph{reachable
set} is $\mathcal{R}=\{\mu(r):r\in\mathcal{S}\}$ and its \emph{diameter} is
$\operatorname{diam}\mathcal{R}=\max_{r}\mu(r)-\min_{r}\mu(r)$.
\end{definition}

\begin{definition}[Resolution]\label{def:res}
A team compares two recipes by running each $k$ times and taking the difference of means, so the
contrast has standard error $\sigma_c=\sigma\sqrt{2/k}$. Its \emph{resolution} $\delta$ is the
smallest true gap it will declare at its chosen power: at $80\%$ power and two-sided $5\%$,
$\delta = 2.80\,\sigma_c$. Resolution is a property of the decision procedure and the number of
runs bought, not of the system.
\end{definition}

\noindent Search buys something only when $\operatorname{diam}\mathcal{R}>\delta$. Below that the
ranking a team recovers is their own noise, and no further search repairs it: the search is not
\emph{underpowered} but \emph{unresolvable}. The failure mode this predicts is specific, so it
is worth stating exactly.

\begin{proposition}[The recipe-search wall]\label{prop:wall}
Let $\hat r=\arg\max_{r\in\mathcal{S}} Y(r)$ be the winner a search returns, and let
$\hat\Delta = Y(\hat r)-Y(\check r)$ be its margin over the observed worst $\check r$. Write
$\Delta=\operatorname{diam}\mathcal{R}$ and let $\varepsilon_r$ be Gaussian. Then
\begin{enumerate}[leftmargin=1.6em,itemsep=1pt,topsep=2pt]
\item \emph{(the margin is manufactured)} $\;\mathbb{E}[\hat\Delta]\;\ge\;\Delta
 \;+\; 2\sigma\bigl(1-o(1)\bigr)\sqrt{2\log m}$, so a search over $m$ recipes reports a
 margin that grows with $m$ at fixed $\Delta$, and grows even when $\Delta=0$;
\item \emph{(and it does not replicate)} if $\Delta=0$, an independent repetition of the whole
 search re-selects the same $\hat r$ with probability exactly $1/m$, and preserves the sign
 of the contrast between any fixed pair with probability exactly $1/2$. Both quantities are
 continuous in $\Delta/\sigma_c$, so they degrade smoothly to chance as the wall is
 approached from above.
\end{enumerate}
\end{proposition}

\begin{proof}[Proof sketch]
(i) $\hat\Delta \ge \max_r \varepsilon_r - \min_r \varepsilon_r$ whenever the $\mu$ are constant,
and the expected range of $m$ i.i.d.\ mean-zero Gaussians is $2\sigma\sqrt{2\log m}(1+o(1))$;
adding non-constant $\mu$ can only increase the observed spread in expectation. (ii) At
$\Delta=0$ the $Y(r)$ are exchangeable, so $\arg\max$ is uniform on $\mathcal{S}$ and two
independent searches agree with probability $1/m$; for a fixed pair the contrast is symmetric
about zero, so its sign is a fair coin, and repetitions are independent. Continuity in
$\Delta/\sigma_c$ is immediate from the Gaussian c.d.f.
\end{proof}

\noindent The wall is therefore joint between a system and the instrument that would measure it,
both of its sides are measurable, and Proposition~\ref{prop:wall} names its operational signature:
\emph{not an absence of winners but winners that do not replicate}. A search below the wall still
returns a best recipe, with a margin that looks larger the harder one searched, and re-running the
comparison returns a different one. \S\ref{sec:budget} reports exactly that: a winner at $0.062$
over $m=6$ pairs, reading $+0.010$ and $-0.015$ on two further runs of the identical protocol.

\paragraph{One setting throughout.} A wall is a joint statement about a
system and an instrument, so it is worth saying at once which of each. Every number in this
paper comes from base (non-instruct) checkpoints, supervised fine-tuning under AdamW,
exact-match scoring at $k{=}4$, and competition-mathematics skills whose answers are single
values. Nothing here is measured on code, on dialogue, on multilingual mixtures, on preference
optimisation, or on a memoryless optimiser, and \S\ref{sec:license} states the remaining limits
at the same resolution as the claims.

\paragraph{The patient student.} Why should
$\operatorname{diam}\mathcal{R}$ be small at all? A student receives two textbooks, $A$ and
$B$. Three idealisations: (i)~\emph{patience}, the student may re-read until nothing changes;
(ii) \emph{coherence}, the books never disagree, so no exercise appears in both with different
answers and no convention in one is contradicted by the other; (iii) \emph{room}, the student
never has to forget one thing to hold another. In that limit every recipe question dissolves.
Finish $A$ then $B$, alternate chapters, shuffle pages into one pile: the endpoint is the same
student. $\mathcal{R}$ collapses to a point, so order, arrangement and (holding the pile fixed)
composition are \emph{exactly free}. The only thing never free is how many pages the pile
contains.

Real fine-tuning is not that limit, so $\mathcal{R}$ is a set and the paper's first job is to
measure its span against $\delta$. The largest recipe contrast we observe on any single corpus
is $0.062$, which \emph{clears} the $0.054$ a three-seed design resolves; re-run twice more
under an identical protocol it reads $+0.010$ and $-0.015$. The apparent winner is the search's
own noise, and \S\ref{sec:budget} prices it. Only one of the two quantities is negotiable: at
fixed $\Delta/\sigma_c$, Proposition~\ref{prop:wall} is information-theoretic and no search
strategy crosses it, whereas $\delta$ belongs to the decision procedure, so anything lowering
$\sigma_c$ moves the wall. It is a price list rather than a prohibition, and what it forbids is
certifying a winner below the resolution one has paid for. The rest of the paper breaks each
idealisation on purpose and asks which of
them enlarges $\mathcal{R}$ enough to clear the wall. One does, by an order of magnitude.

This is not merely a parable. A linear learner trained by gradient descent in the lazy regime
\citep{jacot2018ntk,chizat2019lazy} makes each idealisation a hypothesis and each ``dissolves''
a theorem (\S\ref{subsec:theory}). The reframing we take from it is the paper's organising
claim:
\begin{quote} \emph{A measured recipe effect is never a primitive. It is a measurement of which premise of the free-recipe limit broke, equivalently of what enlarged the reachable set.}
\end{quote}
The three premises do not break equally, and a paper giving them equal billing would be
describing its own table of contents rather than its results. What we found, in the order of how
much of it survived its controls: \emph{Disagreement inside one corpus moves a quantity that
accuracy alone cannot see.} Across two different skills, arrangement moves capability:
interleaving beats blocking by $+0.0483$ at $t=8.06$, in $17$ of $18$ (pair, seed) cells, and
$76\%$ of that is the first skill being overwritten, which is forgetting, named as such. What
this paper adds is what happens inside a \emph{single} distribution. Halve one pooled corpus at
random and change only whether the halves write answers under incompatible conventions:
capability stays inside a floor, while the share committed to one convention moves two orders of
magnitude above its own same-convention control, in a direction that control cannot produce, and
writing the convention into the input returns both (\S\ref{sec:null}).

\emph{Patience breaks into a transient rather than a signature.} At finite budget the order-
carrying term grows one power of the budget faster than the symmetric one, so the sign must
move, and it does, twice outside the noise band inside a single run. The effect does not survive
being re-run: the grid's largest order contrast clears the resolution and then reverses on
replication (\S\ref{sec:budget}). Order is not a small effect but an unresolvable one, which is
a different claim with a different remedy.

\emph{Room we could motivate and could not demonstrate.} The predicted sign change appears along
a capacity axis and its own volume-matched control removes it (\S\ref{sec:capacity}). We report
the sequence because the arithmetic that caught it is reusable, and grade the claim
\textsc{mixed} on the scorecard rather than quietly dropping it.
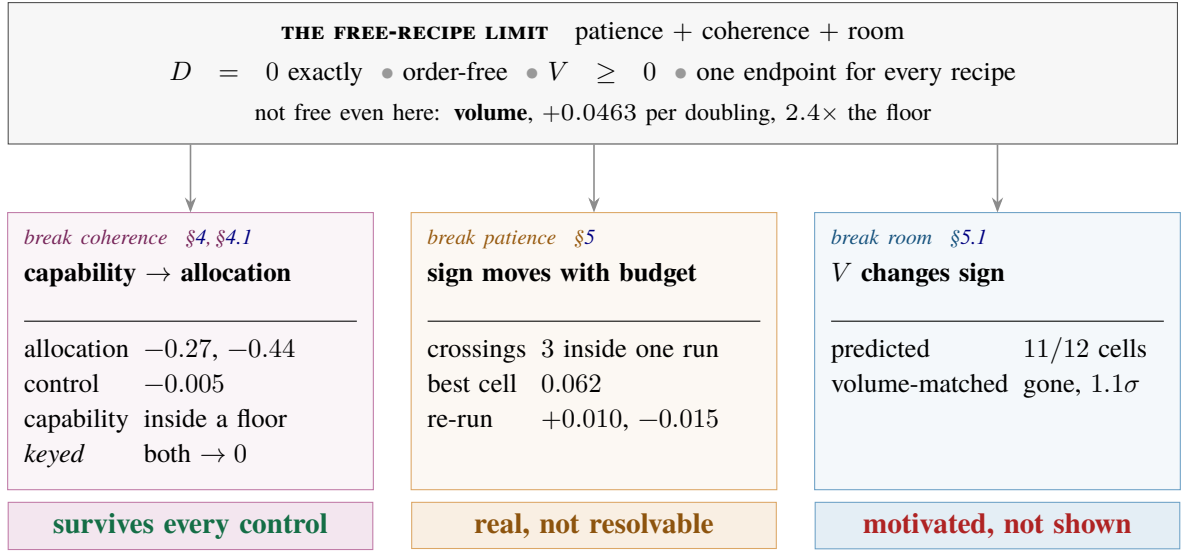
\begin{figure}[t]
\centering 
\begin{tikzpicture}[font=\small,>=Stealth, band/.style={draw=black!70,fill=black!3,sharp corners,inner xsep=8pt,inner ysep=7pt, align=center,text width=14.9cm}, col/.style={draw=none,fill=none,sharp corners,inner xsep=6pt,inner ysep=6pt, align=left,text width=4.4cm,anchor=north west}, vd/.style={draw=#1!65,fill=#1!12,sharp corners,inner xsep=6pt,inner ysep=4pt, align=center,text width=4.4cm,anchor=north west,font=\bfseries}, fl/.style={->,semithick,black!45}] \node[band] (lim) {\textbf{\textsc{the free-recipe limit}}\quad patience $+$ coherence $+$ room\\[2pt] $D=0$ exactly \ \textcolor{black!40}{$\bullet$}\ order-free \ \textcolor{black!40}{$\bullet$}\ $V\ge0$ \ \textcolor{black!40}{$\bullet$}\ one endpoint for every recipe\\[2pt] {\footnotesize not free even here: \textbf{volume}, $+0.0463$ per doubling, $2.4\times$ the floor}}; \node[col=cD,below=9mm of lim.south west,anchor=north west] (c1) {{\scriptsize\itshape\textcolor{cD!75!black}{break coherence \ \S\ref{sec:null},\,\S\ref{sec:conflict}}}\\[1pt] \textbf{capability $\to$ allocation}\\[2pt] \rule{\linewidth}{0.4pt}\\[3pt]
\begin{tabular}{@{}l@{\;\;}l@{}} allocation & $-0.27$, $-0.44$\\ control & $-0.005$\\ capability & inside a floor\\ \emph{keyed} & both $\to0$\\
\end{tabular}
}; \node[col=cA,right=5mm of c1.north east,anchor=north west] (c2)
{{\scriptsize\itshape\textcolor{cA!75!black}{break patience \ \S\ref{sec:budget}}}\\[1pt]
\textbf{sign moves with budget}\\[2pt] \rule{\linewidth}{0.4pt}\\[3pt]
\begin{tabular}{@{}l@{\;\;}l@{}} crossings & $3$ inside one run\\ best cell & $0.062$\\ re-run & $+0.010$, $-0.015$\\
\end{tabular}
}; \node[col=cV,right=5mm of c2.north east,anchor=north west] (c3)
{{\scriptsize\itshape\textcolor{cV!75!black}{break room \ \S\ref{sec:capacity}}}\\[1pt]
\textbf{$V$ changes sign}\\[2pt] \rule{\linewidth}{0.4pt}\\[3pt]
\begin{tabular}{@{}l@{\;\;}l@{}} predicted & $11/12$ cells\\ volume-matched & gone, $1.1\sigma$\\
\end{tabular}
}; 
fit, so the three verdicts stay comparable at a glance however the column text changes.
\node[fit=(c1)(c2)(c3),inner sep=0pt,draw=none] (cols) {}; 
above are top-aligned and of unequal depth, so the 
down to the common lower edge. Ragged box bottoms 
accident; this is the fix that does not require 
need.
\begin{pgfonlayer}{background} \foreach \n/\c in {c1/cD,c2/cA,c3/cV}{ \filldraw[fill=\c!5,draw=\c!65] (\n.north west) rectangle (\n.south east |- cols.south);}
\end{pgfonlayer}
\node[vd=cD,anchor=north west,text=cOK] (v1) at ([yshift=-1.6mm]c1.north west |- cols.south) {survives every control}; \node[vd=cA,anchor=north west,text=cA!70!black] (v2) at (c2.north west |- v1.north) {real, not resolvable}; \node[vd=cV,anchor=north west,text=cNO] (v3) at (c3.north west |- v1.north) {motivated, not shown}; \draw[fl] (lim.south -| c1.north) -- (c1.north); \draw[fl] (lim.south -| c2.north) -- (c2.north); \draw[fl] (lim.south -| c3.north) -- (c3.north);
\end{tikzpicture}
\caption{\textbf{One limit, three ways out of it, three different fates.} Each column breaks one premise \emph{on purpose}; its second line is what the theory says should then move, its table the measurement, its last row the verdict its own controls returned. Breaking coherence does not switch an absent effect on: it changes \emph{which quantity} arrangement moves, and a key written into the input at training time returns both to zero. Allocation is quoted against its own same-convention control, capability and the floors in units of the $0.0191$ contrast floor; $0.062$ is the grid's largest order contrast, beside what the identical protocol returned on re-running.}
\label{fig:map}
\end{figure}
\subsection{The limit as theorems}
\label{subsec:theory}
\paragraph{What this apparatus is for.} Figure~\ref{fig:map} is the map: one limit, three ways
out of it, three different fates. The theorems below are a first-order, lazy-regime
account, displayed because they generate the paper's predictions and say which experiment is
worth running, not because they are load-bearing on their own. They buy \emph{signs, existences
and scope boundaries}: which quantity a departure should move, in which direction, under which
broken premise. They do not buy locations or magnitudes, and that is a finding rather than a
hedge, since every location-level prediction we derived from them died
(\S\ref{sec:license}: six deaths against sixteen supported claims).

A skill $X$ is a finite sample set whose inputs span $\Sh_X$, with orthogonal projection $P_X$;
targets decompose as $w^\star_X=\bar w+\kappa_X$ with a shared core $\bar w$ and a
skill-specific convention $\kappa_X$. Coherence is $\kappa_A=\kappa_B$. Gradient descent on
skill $X$ moves parameters only inside $\Sh_X$, and at convergence acts as
$\theta\mapsto(I-P_X)\theta+P_Xw^\star_X$. Write $\delta:=\theta_0-\bar w$,
$W:=\Sh_{A\cup B}\cap\Sh_B^\perp$ (the directions $A$ supplies outside $B$'s span), and let
$e_S$, $e_J$ be the error vectors of the blocked ($\Astep$) and interleaved arms on $B$.

\begin{theorem}[The limit]\label{thm:limit}
Under patience (each stage to convergence) and coherence:
\begin{enumerate}[label=(\alph*),leftmargin=1.6em,itemsep=1pt,topsep=2pt]
\item \emph{Arrangement:} $e_S-e_J=P_W(I-P_A)\,\delta$. This vanishes identically on $B$'s
training set, vanishes exactly for orthogonal skills, and is otherwise doubly suppressed: it
carries only what stage~$1$ \emph{failed} to learn, re-projected into $W$ at cost
$c=\|P_BP_A\|_{\mathrm{op}}$.
\item \emph{Order:} the directional contrast is carried by the commutator $[P_B,P_A]\,\delta$,
whose norm is $\max_i\cos\theta_i\sin\theta_i\le\tfrac12$ over principal angles: zero for
orthogonal skills \emph{and} for nested ones. Under coherent targets the recency terms it
compares are both $O(\varepsilon)$ in the unexplored mass, so order is free in the limit.
\item \emph{Composition:} without a capacity constraint, $e_O-e_J=P_W\delta$: co-training can
only add directions, so its value $V\ge0$. Observed $V<0$ at matched budget is itself evidence
that some premise broke.
\end{enumerate}
\end{theorem}

\begin{figure}[t]
\centering
\begin{tikzpicture}[scale=1.0,>=Stealth,line join=round,font=\small]
\begin{scope}
 \node[anchor=south west,font=\bfseries\small] at (0,4.9) {(a) the two sample spans};
 \fill[cV!12] (0.1,0.9) -- (4.3,0.9) -- (5.0,1.75) -- (0.8,1.75) -- cycle;
 \draw[cV!60] (0.1,0.9) -- (4.3,0.9) -- (5.0,1.75) -- (0.8,1.75) -- cycle;
 \node[cV!80!black,anchor=west] at (0.25,1.32) {$\Sh_B$};
 \fill[cA!14] (1.6,0.35) -- (3.2,0.35) -- (4.5,4.3) -- (2.9,4.3) -- cycle;
 \draw[cA!70] (1.6,0.35) -- (3.2,0.35) -- (4.5,4.3) -- (2.9,4.3) -- cycle;
 \node[cA!80!black,anchor=east] at (2.82,3.95) {$\Sh_A$};
 \fill[cD!25] (2.35,1.78) -- (3.85,1.78) -- (4.45,4.25) -- (2.95,4.25) -- cycle;
 \draw[cD!70,densely dashed] (2.35,1.78) -- (3.85,1.78) -- (4.45,4.25) -- (2.95,4.25) -- cycle;
 \node[cD!85!black] at (3.4,2.9) {$W$};
 \node[cD!85!black,font=\scriptsize,anchor=north west,align=left] at (4.55,4.35)
 {$W=\Sh_{A\cup B}\cap\Sh_B^{\perp}$};
 \node[cD!85!black,font=\scriptsize,anchor=north west,align=left,text width=2.4cm] at (4.6,3.78)
 {the only place transfer can live};
 \draw[cGrey,thin] (2.55,1.32) -- (2.95,2.05);
 \node[cGrey,font=\scriptsize,anchor=east] at (2.62,1.78) {$c$};
 \node[cGrey,font=\scriptsize,anchor=west] at (0.0,0.1)
 {$c=\|P_BP_A\|_{\mathrm{op}}$ (cosine of the \emph{smallest} principal angle)};
\end{scope}
\begin{scope}[shift={(8.6,0)}]
 \node[anchor=south west,font=\bfseries\small] at (-1.1,4.9) {(b) where the three arms land};
 \coordinate (th0) at (1.60,3.60);
 \coordinate (eO) at (0.30,1.05);
 \coordinate (eJ) at (2.70,1.05);
 \coordinate (eS) at (3.65,1.30);
 \fill[cD!8] (-0.30,0.55) rectangle (4.20,1.95);
 \draw[cD!45,densely dashed] (-0.30,0.55) rectangle (4.20,1.95);
 \node[cD!85!black,font=\scriptsize,anchor=south east] at (4.20,2.02) {$W\subset\Sh_B^{\perp}$};
 \draw[cGrey,->,thin] (th0) -- (eO);
 \draw[cGrey,->,thin] (th0) -- (eJ);
 \draw[cGrey,->,thin] (th0) -- (eS);
 \fill[cGrey] (th0) circle (1.9pt);
 \node[cGrey,anchor=south] at ($(th0)+(0,0.12)$) {$\theta_0$};
 \fill[cGrey] (eO) circle (1.9pt);
 \fill[cV] (eJ) circle (1.9pt);
 \fill[cD] (eS) circle (1.9pt);
 \draw[cV,line width=1.5pt,->] (eO) -- (eJ);
 \draw[cD,line width=1.5pt,->] (eJ) -- (eS);
 \node[cV!85!black,font=\scriptsize,anchor=north] at (1.30,0.96) {$V=P_W\delta$};
 \node[cD!85!black,font=\scriptsize,anchor=north] at (3.38,0.96) {$D=P_Wu$};
 \node[font=\scriptsize,anchor=north,align=center] at (0.30,0.42) {\only\\B alone};
 \node[font=\scriptsize,anchor=north,align=center] at (2.70,0.42) {\shuf\\interleaved};
 \node[font=\scriptsize,anchor=west,align=left] at (4.32,1.30) {\blocked\\$\Astep$};
\end{scope}
\node[anchor=north west,font=\scriptsize,align=left,text width=15.3cm] at (-0.1,-0.35)
 {Both contrasts live in the same subspace $W$ and are gated by the same unexplored mass
 $\varepsilon_B$, but $D$ pays two factors $V$ does not: it sees only A's \emph{unlearned
 residual} $u=(I-P_A)\delta$, and that residual is orthogonal to $\Sh_A$, so re-entering $W$
 costs a further factor $c$. Lengths are schematic: at matched volume the
 measured $|D|/|V|$ is $1.39$, i.e.\ the suppression this panel illustrates is not what the
 data shows (\S\ref{sec:null}, \S\ref{sec:license}).};
\end{tikzpicture}
\caption{\textbf{Why arrangement is doubly suppressed in the limit.} (a) Each skill's gradient
updates are trapped in the span of its own training inputs, so the only directions through which
$A$ can change $B$'s population risk are those $A$ supplies \emph{outside} $B$'s span,
$W=\Sh_{A\cup B}\cap\Sh_B^{\perp}$ (dashed). (b) Composition value $V$ rides on the full leakage
$P_W\delta$, while the arrangement contrast $D$ rides on $P_W(I-P_A)\delta$, the same subspace
but only what stage~1 \emph{failed} to learn, projected back in at cost $c$: two multiplicative
penalties, neither of which goes away with more data. Neither of the two measurements that would
confirm that ordering supports it (at matched volume $|D|/|V|=1.39$, and the sufficient
condition is unreachable for subspaces from one base model, \S\ref{sec:license}), so the panel
states the mechanism the theory proposes, not the measurement.}
\label{fig:geometry}
\end{figure}

Figure~\ref{fig:geometry} draws why arrangement is doubly suppressed in that limit. The proofs
are four lines each, compositions of the convergence map, and every statement, along
with the finite-budget and conflict extensions below, is verified numerically by $25$ machine
checks with hard-coded tolerances. Four of the twenty-five \emph{rejected drafts of our own
statements}, and the versions here are the post-rejection ones
(\S\ref{sec:license}).\footnote{The itemised table is in
the project repository together with the full scorecard.}

Two extensions turn the limit into predictions about its own failure. First, at finite budget
$T$ each projection becomes the spectral filter $M_X(T)=(I-\eta\Sigma_X)^T$, and the
arrangement operator $\Delta(T)=M_B M_A-M_{A\cup B}^2$ expands (small $\eta$, equal sample
sizes) into
\begin{equation}\label{eq:expansion}
\Delta(T)\;\approx\;\underbrace{-\tfrac{\eta^{2}T}{4}\,(\Sigma_A-\Sigma_B)^{2}}
_{\text{symmetric, }O(T)}
\;+\;\underbrace{\tfrac{\eta^{2}T^{2}}{2}\,[\Sigma_B,\Sigma_A]}
_{\text{order-carrying, }O(T^{2})},
\end{equation}
so the only order-carrying term grows one power of $T$ faster than the symmetric one, and an
effect set by their competition \emph{must} change sign along the budget axis
(Figure~\ref{fig:transient}; in simulation the sign flips in $72\%$ of random non-commuting
pairs). Second, if coherence fails ($\kappa_A\neq\kappa_B$), the interleaved arm converges to a
sample-weighted compromise on the contested directions while the blocked arm's final stage takes
them outright, giving
\begin{equation}\label{eq:conflict}
D_{\mathrm{bulk}}=-\Big\|\Sigma_B^{1/2}P_{\mathrm{con}}\,
\tfrac{n_A}{n_A+n_B}(\kappa_B-\kappa_A)\Big\|^{2}<0:
\end{equation}
under disagreement, blocked training must \emph{win}. Both are sign and existence predictions
only, and their epistemic status should be stated as plainly as their content. Full SFT is
measurably not lazy (activation-subspace overlap $0.88$--$0.90$ against LoRA-$r8$'s $0.985$,
\S\ref{sec:conclusion}) and the optimiser used throughout is not memoryless
(\S\ref{sec:budget}), so Eqs.~\eqref{eq:expansion}--\eqref{eq:conflict} are formalised
intuition with an exact domain, not quantitative models of the experiment. Nothing in this
paper derives a magnitude.

\begin{figure}[t]
\centering
\begin{tikzpicture}[font=\small,>=Stealth]
\begin{scope}
 \node[anchor=south west,font=\bfseries\scriptsize] at (-0.55,3.35)
 {(a) the arrangement effect is a transient};
 \draw[cGrey,thin] (-0.25,0) -- (6.4,0);
 \draw[cGrey,thin] (-0.25,0) -- (-0.25,3.15);
 \node[cGrey,font=\scriptsize,rotate=90,anchor=south] at (-0.85,1.6)
 {risk(interleaved) $-$ risk(blocked)};
 \node[cGrey,font=\scriptsize,anchor=north] at (3.1,-0.55) {budget $T$ (log scale)};
\draw[cD,line width=1.2pt] (0.000,0.611) -- (0.406,1.095) -- (0.944,1.970) -- (1.350,2.277) -- (1.756,1.524) -- (2.294,0.145) -- (2.700,0.001) -- (3.106,0.000) -- (3.644,0.000);
\fill[cD] (0.000,0.611) circle (1.9pt);
\fill[cD] (0.406,1.095) circle (1.9pt);
\fill[cD] (0.944,1.970) circle (1.9pt);
\fill[cD] (1.350,2.277) circle (1.9pt);
\fill[cD] (1.756,1.524) circle (1.9pt);
\fill[cD] (2.294,0.145) circle (1.9pt);
\fill[cD] (2.700,0.001) circle (1.9pt);
\fill[cD] (3.106,0.000) circle (1.9pt);
\fill[cD] (3.644,0.000) circle (1.9pt);
\draw[cGrey,densely dashed] (1.350,0) -- (1.350,2.277);
\node[cD!85!black,font=\scriptsize,anchor=south,align=center] at (1.350,2.397) {interior maximum\\at $T=10$};
\draw[cGrey,thin] (0.000,0) -- (0.000,-0.090);
\node[cGrey,font=\scriptsize,anchor=north] at (0.000,-0.12) {1};
\draw[cGrey,thin] (0.406,0) -- (0.406,-0.050);
\draw[cGrey,thin] (0.944,0) -- (0.944,-0.090);
\node[cGrey,font=\scriptsize,anchor=north] at (0.944,-0.12) {5};
\draw[cGrey,thin] (1.350,0) -- (1.350,-0.050);
\draw[cGrey,thin] (1.756,0) -- (1.756,-0.090);
\node[cGrey,font=\scriptsize,anchor=north] at (1.756,-0.12) {20};
\draw[cGrey,thin] (2.294,0) -- (2.294,-0.050);
\draw[cGrey,thin] (2.700,0) -- (2.700,-0.090);
\node[cGrey,font=\scriptsize,anchor=north] at (2.700,-0.12) {100};
\draw[cGrey,thin] (3.106,0) -- (3.106,-0.050);
\draw[cGrey,thin] (3.644,0) -- (3.644,-0.090);
\node[cGrey,font=\scriptsize,anchor=north] at (3.644,-0.12) {500};
 \node[cGrey,font=\scriptsize,anchor=west,text width=3.0cm] at (3.35,2.35)
 {zero at $T{=}0$, off a log axis, and again once both arms reach the same projection};
\end{scope}
\begin{scope}[shift={(8.25,0)}]
 \node[anchor=south west,font=\bfseries\scriptsize] at (-0.5,3.35)
 {(b) which term dominates depends on $T$};
 \draw[cGrey,thin] (-0.25,0) -- (5.5,0);
 \draw[cGrey,thin] (-0.25,0) -- (-0.25,3.15);
 \node[cGrey,font=\scriptsize,rotate=90,anchor=south] at (-0.85,1.6)
 {$\|$commutator$\|$ / $\|$dissimilarity$\|$};
 \node[cGrey,font=\scriptsize,anchor=north] at (2.4,-0.55) {budget $T$ (log scale)};
\draw[cA,line width=1.2pt] (0.000,0.085) -- (0.950,0.171) -- (1.900,0.342) -- (2.850,0.684) -- (3.800,1.367) -- (4.750,2.735);
\fill[cA] (0.000,0.085) circle (1.9pt);
\draw[cGrey,thin] (0.000,0) -- (0.000,-0.09);
\node[cGrey,font=\scriptsize,anchor=north] at (0.000,-0.12) {1};
\fill[cA] (0.950,0.171) circle (1.9pt);
\draw[cGrey,thin] (0.950,0) -- (0.950,-0.09);
\node[cGrey,font=\scriptsize,anchor=north] at (0.950,-0.12) {2};
\fill[cA] (1.900,0.342) circle (1.9pt);
\draw[cGrey,thin] (1.900,0) -- (1.900,-0.09);
\node[cGrey,font=\scriptsize,anchor=north] at (1.900,-0.12) {4};
\fill[cA] (2.850,0.684) circle (1.9pt);
\draw[cGrey,thin] (2.850,0) -- (2.850,-0.09);
\node[cGrey,font=\scriptsize,anchor=north] at (2.850,-0.12) {8};
\fill[cA] (3.800,1.367) circle (1.9pt);
\draw[cGrey,thin] (3.800,0) -- (3.800,-0.09);
\node[cGrey,font=\scriptsize,anchor=north] at (3.800,-0.12) {16};
\fill[cA] (4.750,2.735) circle (1.9pt);
\draw[cGrey,thin] (4.750,0) -- (4.750,-0.09);
\node[cGrey,font=\scriptsize,anchor=north] at (4.750,-0.12) {32};
\draw[cGrey,densely dashed] (-0.15,0.135) -- (4.950,0.135);
\node[cGrey,font=\scriptsize,anchor=west] at (5.000,0.135) {ratio $=1$};
 \node[cA!85!black,font=\scriptsize,anchor=west,text width=2.3cm] at (0.15,2.45)
 {above the line the order-carrying term rules};
\end{scope}
\end{tikzpicture}
\caption{\textbf{What breaking patience does, computed rather than assumed.} Both panels run
gradient descent step by step in the linear model. (a) The blocked-versus-interleaved gap
vanishes at both ends of the budget axis and peaks in between: a mid-training phenomenon, which
is why endpoint-only designs report it inconsistently. (b) The order-carrying commutator term
of Eq.~\eqref{eq:expansion} grows one power of $T$ faster than the symmetric term, so the
balance between them, and with it the sign of the order effect, is a function of the budget.}
\label{fig:transient}
\end{figure}
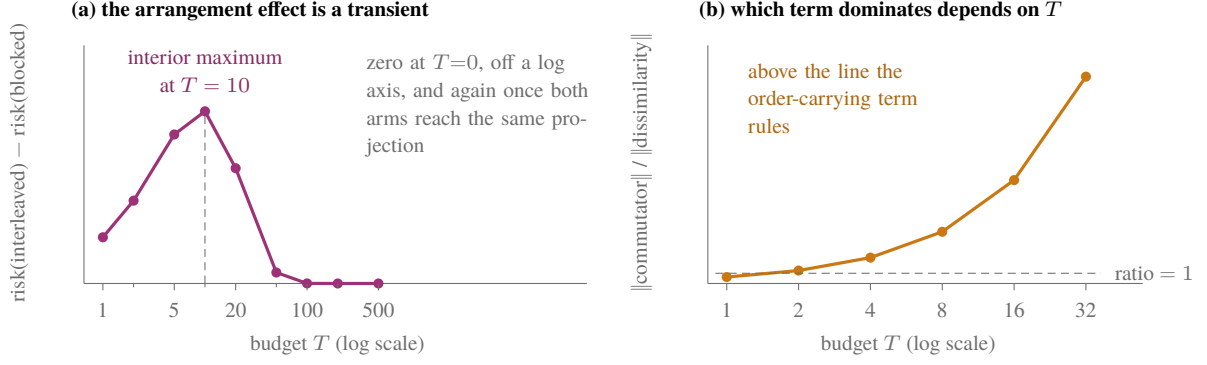

\paragraph{What this framework owes its neighbours.} That the commutator is the right object for
order is not ours: \citet{sweeney2026liebracket} predicts pairwise order at $98.1\%$ accuracy
from a Lie bracket of gradient fields at one step per domain, and their own accuracy decay with
budget ($93.0\%\to81.5\%\to65.3\%$ at $1/20/50$ steps per domain in pretraining,
$98.1\%\to73.1\%$ at $1/20$ in instruction SFT) already \emph{shows} budget dependence. What we
add is narrower: the sign reverses \emph{within} a single run, Eq.~\eqref{eq:expansion} predicts
that it must, and the consequence that an order claim without a stated budget names one point on
a curve follows from the mechanism rather than from the decay alone. The projection algebra
itself is classical \citep{evron2022catastrophic}; new here is the \emph{comparison across
arms}, where the blocked-versus-interleaved difference is an identity rather than a bound and
the conflict term of Eq.~\eqref{eq:conflict} is a switch with a testable sign.

\paragraph{What a reported pairwise accuracy also claims, unstated, about its own labels.}
Deferring to that decay is not the only thing our measurements let us do with it. A pairwise
order accuracy is scored against a \emph{measured} sign, and a measured sign is a random variable
whose reproducibility we have now priced. Write $p$ for how often a predictor agrees with the
true sign and $q$ for how often one training run recovers it; the reported agreement is
$A=pq+(1-p)(1-q)$, which is maximised at $p=1$, so $q\ge A$ always. \textbf{A reported accuracy is
therefore a lower bound on the reproducibility of the reporter's own labels, whether or not they
measured it.} Read that way, $98.1\%$ at one step per domain asserts that those order signs
re-run identically at least $96.3\%$ of the time, and $73.1\%$ at twenty steps asserts at least
$60.7\%$. We measured the same quantity directly at a comparable SFT budget, over all six skill
pairs with independent re-runs of the identical protocol (\S\ref{sec:budget}): \textbf{four of six
pairs keep their sign, $66.7\%$, which caps a perfect predictor at $78.9\%$.}

Two readings survive and we can separate neither, which is the point. Their labels may simply be
more reproducible than ours, at a smaller budget on different corpora, in which case the decay
they report is the budget dependence they read it as and Eq.~\eqref{eq:expansion} explains its
mechanism. Or the decay is the reliability of the label falling as the budget grows, in which
case a predictor's score is tracking its criterion's noise and needs no mechanism at all. The
implied floor on their own labels does fall from $96.3\%$ to $60.7\%$ across the range they
report, which is what the second reading predicts and the first does not require.
\emph{This costs us as much as them}: our own Magnus-truncation account of that decay now has a
competitor we cannot exclude, and we state it here rather than keep the tidier version. What
separates them is one number that no order study we know of reports, ours included until
\S\ref{sec:budget}: re-run the same arms and publish how often the sign survives. The
idealisation-first structure, stating the limit in which the phenomenon is exactly zero and then
measuring departures, follows the Capability Convergence Hypothesis \citep{chen2026cch}, and the
debt is structural rather than stylistic: CCH reads inference through a bounded state channel
between an unbounded stream and a prediction, with \emph{disagreement in the stream} making the
choice of access channel visible in behaviour. That premise returns here with the parameters as
the state and the data path as the stream, and \S\ref{sec:conflict}'s switch is its
training-time signature. \acompanion\ develops the channel-level account in full
\citep{dpd2026}, cited for provenance and never for evidence: every number and every control in
what follows is in this paper or its appendices.

\section{Related Work}\label{sec:related}

\emph{Sequential training and forgetting.} This is the line the paper has to answer to. Our
headline half is that blocking loses capability to the skill it trains first, which is
catastrophic interference \citep{mccloskey1989catastrophic,french1999catastrophic}, forty years
old, with a literature of mitigations \citep{kirkpatrick2017ewc,lopezpaz2017gem}, transfer
metrics \citep{diazrodriguez2018metrics}, LLM-era measurements
\citep{wu2024cllmsurvey,conklin2026forgetting,forgettingmech2026}, an alternating-projection
analysis that is our own mechanism \citep{evron2022catastrophic}, and a non-monotone dependence
on task similarity \citep{lee2021continual}. Reporting $+0.0483$ at $t=8$ and stopping would be
a rediscovery, and we say so where we report it.

The claim is what the same contrast does next. Forgetting is a statement about \emph{sequence}:
train two things in order and the earlier one degrades. It has no parameter for whether the two
things \emph{disagree}, and so no way to predict that the capability term falls under resolution
exactly when they do while a second term, which convention the model commits to, takes over and
moves two orders of magnitude further; nor can it predict the control that decides the case, a
key written into the input at training time, changing no schedule and no budget, returning both
terms to zero. Similarity-dependent accounts \citep{lee2021continual} come closest, since
disagreement is a kind of similarity, but they predict a magnitude varying with overlap rather
than a change in \emph{which quantity} the arrangement contrast moves. The distinction is
testable and we tested it: our domain-boundary row makes two sources maximally dissimilar with
nothing to dispute and leaves arrangement inert at $0.5\sigma$, while numeral-against-spelled
makes them maximally similar and fires at $5.16$ floors. Dissimilarity is not the axis;
disagreement is. That is what is new, together with the cross-arm identity that lets both
regimes be read on one coordinate.

\emph{Curriculum and interleaving.} The divided empirical record
\citep{bengio2009curriculum,rohrer2007interleaving,kornell2008interleaving,
beyondrandom2025,goodcurriculum2025,curriculum2026dynamics,interleaved2025,
implicitcurriculum2026,lrdecay2025curriculum} is, on this account, what sampling a transient at
one endpoint per study must produce; the missing covariate is the budget, not a better
difficulty metric. A meta-analysis of human interleaved practice finds it generally beats
blocking, with between-category similarity as the moderator \citep{brunmair2019interleaving};
our coherent grid reproduces that direction, and our conflict rows show what the moderator
becomes when similarity is pushed to contestation of the same inputs, namely that the sign flips
and the moved quantity changes.

\emph{Order prediction.} \citet{sweeney2026liebracket} reaches the commutator first and escapes
the $n!$ search first (\S\ref{subsec:theory}); our contribution is to give their score a
measurable domain of validity. A one-step gradient score is the leading local term, which
presumes that term keeps its sign over the horizon the ordering will be used on, and
\S\ref{sec:budget} measures that presumption directly: the sign turns over three times inside a
single run, and their own decay from $93.0\%$ to $65.3\%$ as pretraining steps per domain rise
from one to fifty is the same horizon becoming visible from the inside. Rank with their scores
if you like, but an ordering is licensed only at a stated budget, and the volume matching no
published multi-source order study reports (ours included, until it inverted our own conclusion)
is a precondition for certification rather than a refinement of it. The same author's
optimiser-memory result \citep{sweeney2026shuffle} supplies at once the sharpest alternative
reading of our transient and the validation of our memoryless mechanism, engaged as a confound
in \S\ref{sec:budget} and \S\ref{sec:conflict}. Concurrent perturbation-window work
independently finds order sensitivity that exists but resists localisation
\citep{pairwisefragile2026}.

\emph{Data mixing and recipe search.} Mixing work optimises proportions
\citep{doremi,doge,gu2025datamixphase,midtraining2025survey}; we add that at fixed proportions
and volume the remaining recipe freedoms are inside noise unless the mixture disagrees with
itself. The word ``recipe'' now names a second object: concurrent work searches executable
pipelines of selection, filtering and mixing operators over a raw pool
\citep{wu2026recipesearch,datachef2026}, whereas ours hold the pool fixed. The searches are
complementary, but Proposition~\ref{prop:wall} prices the step both end with, certifying a
winner, which places the wall in a statistics literature older than this problem: searches over
near-equivalent alternatives return confident winners whose margins are the selection
\citep{sculley2018winnerscurse,cooper2021hpo,bastani2026winner}, and seed-level replication
studies document the per-run noise funding those margins \citep{bui2025seeds}.
\citet{dataorg2026} report that reorganising existing training data improves stability and
performance, which reads as the opposite of our arrangement null and is not: their interventions
act at sample level across a large heterogeneous mixture and are evaluated as single endpoints,
ours at stage level within one coherent domain at matched per-source volume, with each winner
re-run. What our result forbids is expecting such gains \emph{inside one coherent domain at
fixed volume}, and trusting any of them without re-running the winner.

\emph{Disagreement is studied on three sides, and ours is the fourth.}
\citet{dsouza2025disagreement} catalogue where instruction-tuning labels disagree;
\citet{coninstruct2026} ask whether a model \emph{notices} conflicting constraints at inference
time; \citet{incompletelearning2026} and \citet{vcd2026} take conflict as something to be
\emph{resolved}. All three treat disagreement as a defect to be found, flagged or repaired.
\S\ref{sec:conflict} asks the question none of them does: holding the disagreement fixed and
repairing nothing, what does the \emph{arrangement} of the same rows do? Because the answer is
that it changes which quantity arrangement moves, we read allocation and capability separately
rather than reporting one accuracy. Closest in shape to our dose axis is the proportion-varying
line in the knowledge-conflict literature \citep{gekhman2024newknowledge}, which varies how much
of a fine-tuning set contradicts what the model already holds; ours varies how much of a corpus
contradicts \emph{itself}, and adds a learnability guard, since a minority convention below its
own learnability edge produces a null that means nothing. Per-step gradient-conflict methods
\citep{yu2020pcgrad,liu2021cagrad,navon2022nashmtl} resolve contested directions inside the
optimiser step; our conflict lives at the stage timescale, and our separating cell shows none of
the stage-level effect is batch-level.

\emph{The family.} The Capability Convergence Hypothesis \citep{chen2026cch} sets
inference-time capability by access structure, distinguishing a compressive state channel, a
verbatim index channel and the access-complete class holding both, with disagreement in the
stream as what makes the channel choice visible. This paper is its recipe-level member,
and \citet{dpd2026} the path-level one. All three state one practice: an idealisation with exact
premises, pre-registered departures, and a public scorecard.

\emph{Which paper the switch belongs to, since two report it.} The conflict switch appears in this paper
and in \citet{dpd2026}, and a reader is entitled to know which claim rests on which. \textbf{This
paper owns its \emph{extent}}: that the switch exists at all, that it is the one departure with a
systematic answer among the three classical freedoms, and how far it reaches. Everything about
reach is measured here and nowhere else, namely four scales in two pretraining families, three
independent convention constructions, a dose axis, matched same-convention controls at every
point, and the difficulty and domain-boundary nulls that show the trigger is disagreement rather
than heterogeneity. \textbf{\citet{dpd2026} owns its \emph{mechanism}}: why blocking moves
allocation rather than capability, which it derives from an averaging limit in which every block
length collapses to one point and only the stopping phase survives, and which we use here as an
explanation but do not test. Neither result is the other's evidence. If the mechanism paper's
averaging limit were wrong, the rows in Table~\ref{tab:conflict} would stand and lose their
explanation; if these rows failed to replicate, that paper would keep its theorem and lose its
subject. We state the split rather than let the overlap read as two independent confirmations of
one thing, which is exactly what a family of papers is tempted to let happen.

\section{The Instrument}\label{sec:design}

\paragraph{Design.} For an ordered skill pair $(A,B)$ we train five arms (\only\ for each skill,
both blocked orders, and the interleaved shuffle of the union) and evaluate every arm on both
skills: ten measurements per pair. The contrasts map one-to-one onto Theorem~\ref{thm:limit}:
\begin{align*}
V &:= \mathrm{acc}_B(\shuf)-\mathrm{acc}_B(\text{B \only}) &&\text{composition,}\\
D &:= \mathrm{acc}_B(\shuf)-\mathrm{acc}_B(\Astep) &&\text{arrangement ($D<0$: blocked wins),}\\
\ASYM &:= \mathrm{TOTAL}(\Astep)-\mathrm{TOTAL}(\text{B}\!\to\!\text{A}) &&\text{order,}
\end{align*}
with $\mathrm{TOTAL}$ the sum of both skills' accuracies on a sequential arm.

\paragraph{Models, skills, protocol.} Figure~\ref{fig:appdesign} lays out the instrument, five training arms and ten measurements. $12$ base (non-instruct) models from three families
(Qwen3 at $1.7/4/8/14$B, Qwen2.5 at $0.5/1.5/3/7/14$B, Llama-3 at $1/3/8$B
\citep{qwen3,qwen25,llama3herd}), chosen base because instruction-tuned checkpoints have already
been through an unknown recipe, the confound under study. Skills are single-topic slices of
competition mathematics (algebra, calculus, geometry, combinatorics, plus physics) from
OpenR1-Math \citep{openr1math,hendrycks2021math}; training is full-parameter SFT
\citep{zhao2024swift} under the framework's default AdamW ($\beta_1{=}0.9$, $\beta_2{=}0.999$;
the optimiser is stateful, which \S\ref{sec:budget} confronts directly), except where a capacity
constraint is imposed via LoRA rank \citep{hu2022lora}; the interleaved corpus is re-drawn per
training seed, so seed-to-seed spread includes data-order variation; evaluation is $k{=}4$
samples on $120$ held-out problems per skill \citep{kwon2023vllm}, exact match on the boxed
answer. The grid analysed here is $761$ fine-tuning runs and $1922$ scored evaluations, drawn
from a released base of $1233$ runs and $3046$ evaluations that also contains the other
experiments of this line; both figures are reported so that a reader checking the archive
against the paper finds the larger number explained. \emph{Counting rule:} a run is one directory
holding the output of one trainer invocation, whether or not its checkpoint still exists, since
checkpoints are deleted once their results are recorded; a scored evaluation is one
(checkpoint, target) pair carrying an accuracy at the final checkpoint of its run, so the further
$2277$ intermediate-checkpoint evaluations of the trajectory families are counted separately
rather than folded in, each measuring one run repeatedly. The split between the two figures is by
experiment family, each family assigned to exactly one side and none unassigned; the assignment
is released with the counting script rather than described. Every number is produced by an analyzer
rather than transcribed by hand, and the result base, analyzers, $25$ machine checks and the scorecard are
packaged for release with the paper.

\begin{figure}[H]
\centering
\begin{tikzpicture}[font=\scriptsize,>=Stealth,line join=round,
 corpus/.style={draw=#1!70,fill=#1!12,rounded corners=1pt,minimum width=1.9cm,
 minimum height=0.5cm,inner sep=2pt},
 stage/.style={draw=cGrey!75,fill=cGrey!6,rounded corners=1pt,minimum width=1.5cm,
 minimum height=0.48cm,inner sep=2pt},
 ckpt/.style={draw=black!65,fill=black!5,rounded corners=1pt,minimum width=1.45cm,
 minimum height=0.44cm,inner sep=1pt},
 arm/.style={anchor=east,font=\scriptsize\scshape,cGrey!140,inner sep=1pt},
 obsv/.style={draw=#1,fill=#1!8,rounded corners=1pt,inner sep=4pt,align=left},
 fl/.style={->,cGrey,thin},
 feed/.style={->,draw=#1!80,line width=0.6pt,rounded corners=2.5pt}]

\node[corpus=cV] (ca) at (0.95,5.55) {skill $A$: $622$ rows};
\node[corpus=cA] (cb) at (0.95,4.95) {skill $B$: $622$ rows};
\node[cGrey,font=\scriptsize,anchor=south] at (0.95,5.85) {volume-matched corpora};

\def\cA{3.85}\def\cB{5.75}\def\cC{7.65}
\node[arm] at (2.62,4.05) {only$_A$}; \node[stage] (oa) at (\cA,4.05) {train $A$};
 \node[ckpt] (k1) at (\cC,4.05) {$\theta_{\only_A}$};
\node[arm] at (2.62,3.35) {only$_B$}; \node[stage] (ob) at (\cA,3.35) {train $B$};
 \node[ckpt] (k2) at (\cC,3.35) {$\theta_{\only_B}$};
\node[arm] at (2.62,2.65) {blocked $A\!\to\!B$};
 \node[stage] (s1) at (\cA,2.65) {train $A$}; \node[stage] (s2) at (\cB,2.65) {train $B$};
 \node[ckpt] (k3) at (\cC,2.65) {$\theta_{A\to B}$};
\node[arm] at (2.62,1.95) {blocked $B\!\to\!A$};
 \node[stage] (t1) at (\cA,1.95) {train $B$}; \node[stage] (t2) at (\cB,1.95) {train $A$};
 \node[ckpt] (k4) at (\cC,1.95) {$\theta_{B\to A}$};
\node[arm] at (2.62,1.25) {shuf};
 \node[stage,minimum width=3.4cm] (sh) at (4.8,1.25) {train shuffled $A\cup B$};
 \node[ckpt] (k5) at (\cC,1.25) {$\theta_{\shuf}$};

\draw[feed=cV] (ca.east) -- ++(0.88,0) |- (oa.west);
\draw[feed=cA] (cb.east) -- ++(0.69,0) |- (ob.west);
\draw[feed=cV] (ca.east) -- ++(0.88,0) |- (s1.west);
\draw[feed=cA] (cb.east) -- ++(0.69,0) |- (t1.west);
\draw[feed=cV] (ca.east) -- ++(0.88,0) |- ([yshift=2.6pt]sh.west);
\draw[feed=cA] (cb.east) -- ++(0.69,0) |- ([yshift=-2.6pt]sh.west);
\draw[fl] (oa) -- (k1); \draw[fl] (ob) -- (k2);
\draw[fl] (s1) -- (s2); \draw[fl] (s2) -- (k3);
\draw[fl] (t1) -- (t2); \draw[fl] (t2) -- (k4); \draw[fl] (sh) -- (k5);

\node[draw=cGrey!60,fill=cGrey!4,rounded corners=1pt,minimum height=3.25cm,minimum width=1.5cm,
 align=center] (ev) at (9.6,2.65) {evaluate on\\\emph{both} skills\\[2pt]$k=4$\\$n=120$};
\foreach \k in {k1,k2,k3,k4,k5} \draw[fl] (\k.east) -- (\k.east -| ev.west);
\node[cGrey,font=\scriptsize,anchor=north,align=center] at (9.6,0.92)
 {ten measurements};

\node[obsv=cV,anchor=north west,text width=4.0cm] at (0.0,0.35)
 {$V=\mathrm{acc}_B(\shuf)-\mathrm{acc}_B(\only_B)$\\[1pt]
 \textbf{composition}\\ arms differ in volume:\\ $1244$ against $622$ rows};
\node[obsv=cD,anchor=north west,text width=3.5cm] at (4.63,0.35)
 {$D=\mathrm{acc}_B(\shuf)-\mathrm{acc}_B(A{\to}B)$\\[1pt]
 \textbf{arrangement}\\ same data, same volume};
\node[obsv=cA,anchor=north west,text width=3.9cm] at (8.76,0.35)
 {$\ASYM=\mathrm{TOT}(A{\to}B)-\mathrm{TOT}(B{\to}A)$\\[1pt]
 \textbf{order}\ \ ($\mathrm{TOT}$ sums both skills)\\ same data, same volume};
\end{tikzpicture}
\caption{\textbf{The instrument: five training arms, ten measurements.} Five arms are built from
two volume-matched corpora, and every checkpoint is evaluated on \emph{both} skills, which makes
the design ten measurements rather than five and makes the order contrast computable at all. The
three contrasts differ in one respect that matters: $D$ and $\ASYM$ compare arms that saw
identical data in identical amounts, while $V$'s two arms differ by a factor of two in volume
unless matched deliberately, which is why \S\ref{sec:null} reports composition twice.}
\label{fig:appdesign}
\end{figure}
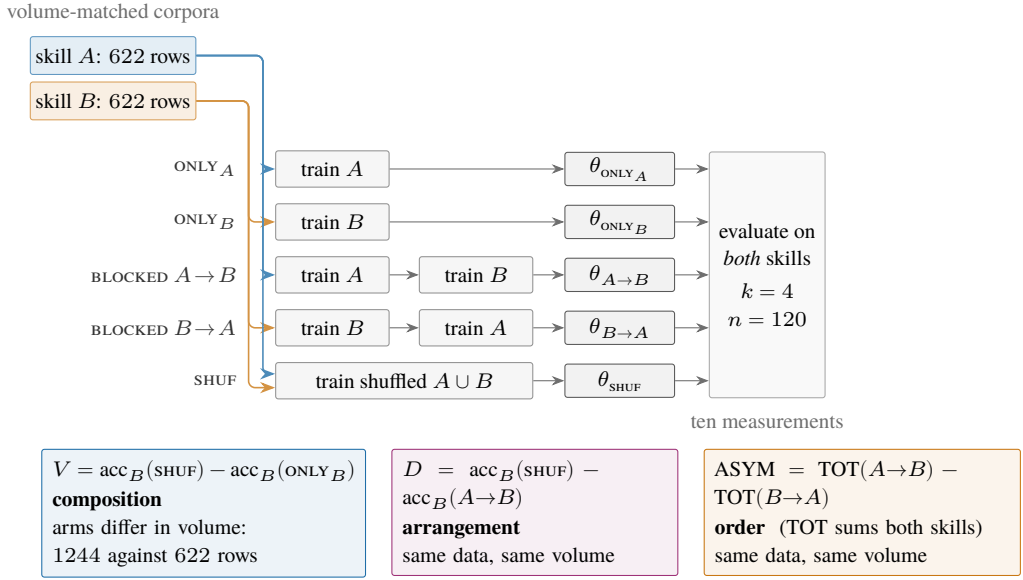

\paragraph{The noise convention.} Single-run seed noise is $\hat\sigma=0.0234$; every
contrast standard error scales it by $\sqrt2$ (same-seed arms share checkpoints and evaluation
sets, so errors correlate positively, which we measured rather than assumed) and divides by
$\sqrt{\text{seeds}}$, giving a \emph{contrast floor} of $0.0191$ at three seeds. The minimum
effect detectable at $80\%$ power against that floor is $\approx2.8$ floors, $0.054$ accuracy
points. Every null below is reported against this MDE, not as a bare $p$-value.

\paragraph{What that noise is made of.} A floor that does not say what
it is measuring prices every remedy at the training rate, so we decompose it. Cross-seed
variation of a single cell is training stochasticity \emph{plus} evaluation sampling at $k{=}4$
generations per problem, and the second term needs no new runs to estimate: the problem set is
held fixed across arms and seeds, making problem sampling a common offset that cannot contribute
to a contrast, and each problem's $k$ pass fractions are stored. With
$\mathrm{acc}=\frac1n\sum_i\hat p_i$ and $\hat p_i=\mathrm{Bin}(k,p_i)/k$,
$\operatorname{Var}(\mathrm{acc}\mid\text{problems})=\frac{1}{n^2}\sum_i p_i(1-p_i)/k$, with
$\hat p(1-\hat p)\,k/(k-1)$ unbiased for $p(1-p)$. Pooled over the $379$ (family, arm, target)
cells with three or more seeds, the cross-seed s.d.\ is $0.0142$, of which evaluation sampling
accounts for $0.0132$: \textbf{$86\%$ of the variance}, leaving training stochasticity at
$0.0053$. Three consequences follow, in increasing order of weight. The frozen
$\hat\sigma=0.0234$ is conservative by $1.6\times$, so every null here is stated against a floor
larger than the instrument's own. Most of the noise inside one protocol is bought down at the
evaluation rate, since quadrupling $k$ costs minutes and would take the MDE from $0.033$ to
$0.019$ (Table~\ref{tab:noisedecomp}) where a fourth training seed costs hours and buys less.
And, binding hardest, \emph{none of it touches the wall}: the run-level s.d.\ of $0.0364$ setting
$\delta_{\text{run}}$ is measured between three replications of a whole protocol, each already a
three-seed mean, and per-cell seed noise predicts only about a fifth of its variance. 

\begin{table}[H]
\centering\small
\caption{\textbf{Most of the noise inside one protocol is evaluation sampling, and it is the
cheap half.} Single-run s.d.\ is the pooled cross-seed value; the contrast floor is that s.d.\
carried through $\sqrt2$ for a two-arm contrast and $\sqrt3$ for three seeds, the paper's
convention. Only the generation term falls with $k$; the training term, $0.0053$, is the
$k\to\infty$ limit and is what more evaluation cannot buy. The paper's own floor is the frozen
$0.0191$, larger than every row, and all nulls are reported against it.}
\label{tab:noisedecomp}
\begin{tabular}{lccc}
\toprule
samples per problem & single-run s.d. & contrast floor & MDE at $80\%$ \\
\midrule
$k=4$ \ \emph{(as run)} & $\mathbf{0.0142}$ & $\mathbf{0.0116}$ & $\mathbf{0.0326}$ \\
$k=8$ & $0.0107$ & $0.0088$ & $0.0246$ \\
$k=16$ & $0.0085$ & $0.0069$ & $0.0194$ \\
$k=32$ & $0.0071$ & $0.0058$ & $0.0162$ \\
$k\to\infty$ \ \emph{(training only)} & $0.0053$ & $0.0043$ & $0.0122$ \\
\bottomrule
\end{tabular}

\end{table}

\noindent Replication
evidence says the floor is conservative for seed noise (pooled within-cell s.e.\ $0.0125$) and
optimistic for run-to-run variation (\S\ref{sec:budget}), so we keep it and read every
``$\times$ floor'' as a lower bound on the true uncertainty.

\paragraph{Volume matching is a precondition.} Natural corpus sizes
($962$ algebra rows against $622$ calculus) confound every order contrast, because flipping the
order also flips which skill gets the longer final stage. On unmatched data our order analysis
gave one ordering; matched to $622$ rows per stage it gave the \emph{opposite} ordering
(\S\ref{sec:budget}). All headline contrasts below are volume-matched, and no published
\emph{multi-source} order study we know of, including \citet{sweeney2026liebracket}, reports
matching per-source volume. That indictment reaches only the design it names: difficulty-sorted
curricula permuting one fixed corpus preserve volume by construction, and the confound is
specific to sequential training over sources of unequal size.

\paragraph{Reliability as instrument calibration.} Before interpreting any contrast we estimate
its $\ICC(1)$ over training seeds \citep{shrout1979icc} on $187$ (model, pair, target) triples
at five seeds (Table~\ref{tab:reliability}): composition $0.803$ ($F{=}14.63$), order $0.524$,
arrangement $0.197$ with $F{=}1.82$ against $F_{.05}{=}2.00$, so \emph{no significant reliable
variance}, its between-cell dispersion half its within-cell noise, its sign a near coin flip
($97{+}/81{-}$, $p{=}0.26$). We treat this ranking as a property of the instrument on this
population rather than a law: the geometry that motivates it does not license it here
(\S\ref{sec:license}), and the ranking is computed on natural corpus sizes. Under volume
matching, $\ASYM$'s reliability moves with the budget, from $0.860$ where training barely moves
the model to $0.149$ where it teaches, so a single ``reliability of order'' does not exist
independently of corpus and budget.

\begin{table}[t]
\centering\small
\caption{\textbf{Instrument calibration: reliability of each observable} ($\ICC(1)$, one-way
ANOVA over training seeds; $187$ triples for $V$/$D$, $105$ for order; natural corpus sizes).
Composition is the only recipe contrast whose structure reliably survives reseeding;
arrangement carries no significant reliable variance. This is what the observables \emph{can}
measure, reported before what they \emph{did} measure.}
\label{tab:reliability}
\begin{tabular}{lcccccc}
\toprule
observable & ICC(1) & $F$ & $F_{.05}$ & $\hat\sigma_{\text{within}}$ & $\hat\sigma_{\text{between}}$ & signif. \\
\midrule
$V$ (co-training value) & 0.803 & 14.63 & 2.00 & 0.0232 & 0.0469 & yes \\
$\mathrm{FWT}=V-D$ & 0.820 & 16.20 & 2.00 & 0.0207 & 0.0442 & yes \\
$\mathrm{ASYM}$ (directional order) & 0.524 & 4.62 & 2.53 & 0.0268 & 0.0281 & yes \\
$\mathrm{RECENCY}(B)$ & 0.481 & 4.04 & 2.53 & 0.0187 & 0.0180 & yes \\
$\mathrm{RECENCY}(A)$ & 0.295 & 2.37 & 2.53 & 0.0204 & 0.0132 & no \\
$D$ (arrangement penalty) & 0.197 & 1.82 & 2.00 & 0.0202 & 0.0100 & no \\
\bottomrule
\end{tabular}

\end{table}

\section{What Arrangement Moves: Forgetting Across Skills, Allocation Within a Contested Corpus}
\label{sec:null}

The free-recipe limit predicts that within a coherent domain, at fixed volume, all three
decisions sit at zero. Measured on the coordinate the contrasts of \S\ref{sec:design} were
defined on, the accuracy of the skill each contrast targets, that is what we see: composition
$0.0101$, arrangement $0.0140$, order $0.0205$ against a $0.0191$ floor
(Table~\ref{tab:matchedthree}), none resolvable at this design's power.

\textbf{That reading does not survive changing the coordinate, and the failure is ours to
report.} Those contrasts score one skill at a time. Scoring the endpoint the way a practitioner
would receive it, on both skills the arms were trained on, which is the coordinate the order
contrast already uses, gives a different and much sharper answer:

\begin{center}\small
\begin{tabular}{lccc}
\toprule
$\shuf$ minus $\blocked$, six volume-matched pairs $\times$ three seeds & mean & $t$ & sign \\
\midrule
first skill (trained first under blocking, then overwritten) & $+0.0428$ & $8.41$ & $17/18$ \\
second skill (trained last under blocking), \emph{the coordinate $D$ scores} & $+0.0132$ & $2.00$ & $14/18$ \\
\textbf{both skills summed} & $\mathbf{+0.0483}$ & $\mathbf{8.06}$ & $\mathbf{17/18}$ \\
\bottomrule
\end{tabular}
\end{center}

Interleaving beats blocking on total capability by $2.5$ floors, in $6$ of $6$ pairs, against
\emph{both} blocked directions, at $t=8$. Arrangement is not free in a coherent domain. The
null we would have reported is an artefact of three choices compounding: scoring one skill,
comparing only against the same-direction blocked arm, and averaging absolute values so that a
systematic sign cancels.

\emph{The mechanism is the one this paper's own retention identity names.} Blocking's cost is
paid by the skill trained \emph{first} and then overwritten: $76\%$ of the total effect sits
there, at $t=8.41$, while the second skill, the only one $D$ scores, moves at $t=2.00$. $D$ was
defined to isolate what arrangement does to the skill trained last, and it does that faithfully;
what it cannot see is forgetting, which is where arrangement actually acts. A contrast can be
correctly computed and still be the wrong coordinate.

\paragraph{The objection this invites.} ``Interleaving
beats blocking because blocking forgets'' is catastrophic forgetting, discovered in $1989$
\citep{mccloskey1989catastrophic}, and a reader is entitled to stop here. We accept the
identification: the result above \emph{is} forgetting, and $76\%$ of it sits on the skill that
gets overwritten.

The answer is a second measurement, built so that it does not compare across distributions. The
tempting test, running the identical contrast on a corpus whose halves disagree, is confounded:
moving from the coherent grid to the conflicted corpora changes coherence \emph{and} whether the
halves contain different content, and forgetting needs the second to have anything to overwrite
(\S\ref{sec:license} gives the argument in full).

\emph{So the load-bearing comparison never leaves one distribution.} Hold the pooled corpus, the
volume, the budget and the scoring fixed, and change only whether its two halves demand
incompatible outputs (Table~\ref{tab:tworegimes}, lower block). Against its own
same-distribution control, conflict leaves capability inside a floor and swings allocation by
two orders of magnitude: $-0.268$ and $-0.443$ against the control's $-0.005$, which is $84$ and
$139$ times the coherent grid's $+0.0032$ and about fifty and ninety times the control's own
value. Writing the convention into the input returns allocation to $0.25\times$ the coherent
anchor and capability to $0.02$ floors on $[-0.63,+0.60]$. Forgetting has no term that switches
on when two sources disagree, none that a key at write time could disable, and nothing to say
about \emph{which} of two conventions a model commits to at fixed total accuracy.

The multipliers here and in Table~\ref{tab:tworegimes} use the \emph{coherent} anchor, on
eighteen observations rather than the control's eight; the two agreed to four decimals at three
seeds and separate by a factor of $1.6$ at eight, and the argument is unchanged under either.
Ratios with denominators this small state a scale rather than measure a multiplier, which is why
the absolute values are reported alongside them.

\begin{table}[H]
\centering\small
\caption{\textbf{The load-bearing comparison is the lower block, which never leaves one
distribution.} Rows 2--5 are random halves of one pooled corpus differing only in whether the
halves demand incompatible outputs and whether a write-time key is present. Row 1 trains two
\emph{different} skills, so it is not a controlled contrast against them; it is how the effect was
found and it calibrates the ratios. Both contrasts are the shuffled arm minus the mean of the
blocked arms. The capability column carries two names because it is not one object across the
regimes (\S\ref{sec:license}): $\mathrm{TOTAL}=\mathrm{acc}_X+\mathrm{acc}_Y$, two skills summed,
on the coherent row, and $\mathrm{SOLVE}=\mathrm{acc}_A+\mathrm{acc}_B$, one problem set under
each convention, on the conflicted ones, with
$\mathrm{SHARE}=\mathrm{acc}_B/(\mathrm{acc}_A+\mathrm{acc}_B)$. Intervals are two-sided $95\%$
$t$ on the seed spread. Qwen2.5-7B at the learning budget; conflicted rows and the control carry
eight seeds, the coherent grid eighteen (pair, seed) cells. A separate code path reproduces $+0.0483$, $t=8.06$, $17/18$
from a separate code path.}
\label{tab:tworegimes}
\begin{tabular}{lccccc}
\toprule
& & \multicolumn{2}{c}{capability} & \multicolumn{2}{c}{allocation} \\
\cmidrule(lr){3-4}\cmidrule(lr){5-6}
corpus & $n$ & contrast & $95\%$ CI, floors & $\Delta\mathrm{SHARE}$ & $\times$ coherent \\
\midrule
\textbf{coherent} --- four maths skills, six pairs & $18$
  & $\mathbf{+0.0483}$ & $\mathbf{[+1.87,+3.19]}$ & $+0.0032$ & $1\times$ \\
\midrule
conflicted --- numeral against spelled & $8$
  & $+0.0016$ & $[-0.61,+0.78]$ & $\mathbf{-0.2677}$ & $\mathbf{84\times}$ \\
conflicted --- synthetic $+1$ on half the answers & $8$
  & $-0.0052$ & $[-1.77,+1.22]$ & $\mathbf{-0.4431}$ & $\mathbf{139\times}$ \\
\midrule
the same conflict, \emph{keyed} at write time & $8$
  & $-0.0003$ & $[-0.63,+0.60]$ & $-0.0008$ & $0.25\times$ \\
control --- both halves numerals, no conflict & $8$
  & $-0.0099$ & $[-1.40,+0.37]$ & $-0.0051$ & $1.6\times$ \\
\bottomrule
\end{tabular}

\end{table}

\begin{figure}[t]
\centering
\resizebox{\linewidth}{!}{%
\begin{tikzpicture}[font=\small,>=Stealth]
\fill[cGrey!12] (0.952,-0.340) rectangle (2.538,2.820);
\draw[cGrey!55,densely dashed] (1.745,-0.340) -- (1.745,2.820);
\draw[cGrey!55,densely dashed] (10.200,-0.340) -- (10.200,2.820);
\node[anchor=east,font=\scriptsize] at (-0.15,2.480) {coherent, six skill pairs \ \textcolor{black!45}{$n{=}18$}};
\draw[cD,line width=0.9pt] (3.228,2.480) -- (4.275,2.480);
\draw[cD,line width=0.9pt] (3.228,2.390) -- (3.228,2.570);
\draw[cD,line width=0.9pt] (4.275,2.390) -- (4.275,2.570);
\fill[cD] (3.751,2.480) circle (2.3pt);
\draw[cD,line width=3.2pt] (10.239,2.480) -- (10.264,2.480);
\node[cD!75!black,font=\scriptsize,anchor=west] at (10.384,2.480) {$+0.0032$};
\node[anchor=east,font=\scriptsize] at (-0.15,1.860) {conflicted: numeral vs spelled \ \textcolor{black!45}{$n{=}8$}};
\draw[cA,line width=0.9pt] (1.261,1.860) -- (2.363,1.860);
\draw[cA,line width=0.9pt] (1.261,1.770) -- (1.261,1.950);
\draw[cA,line width=0.9pt] (2.363,1.770) -- (2.363,1.950);
\fill[cA] (1.816,1.860) circle (2.3pt);
\draw[cA,line width=3.2pt] (10.239,1.860) -- (8.184,1.860);
\node[cA!75!black,font=\scriptsize,anchor=east] at (8.064,1.860) {$-0.2677$};
\node[anchor=east,font=\scriptsize] at (-0.15,1.240) {conflicted: synthetic $+1$ \ \textcolor{black!45}{$n{=}8$}};
\draw[cA,line width=0.9pt] (0.341,1.240) -- (2.712,1.240);
\draw[cA,line width=0.9pt] (0.341,1.150) -- (0.341,1.330);
\draw[cA,line width=0.9pt] (2.712,1.150) -- (2.712,1.330);
\fill[cA] (1.531,1.240) circle (2.3pt);
\draw[cA,line width=3.2pt] (10.239,1.240) -- (6.837,1.240);
\node[cA!75!black,font=\scriptsize,anchor=east] at (6.717,1.240) {$-0.4431$};
\node[anchor=east,font=\scriptsize] at (-0.15,0.620) {the same, \emph{keyed} at write time \ \textcolor{black!45}{$n{=}8$}};
\draw[cV,line width=0.9pt] (1.245,0.620) -- (2.221,0.620);
\draw[cV,line width=0.9pt] (1.245,0.530) -- (1.245,0.710);
\draw[cV,line width=0.9pt] (2.221,0.530) -- (2.221,0.710);
\fill[cV] (1.729,0.620) circle (2.3pt);
\draw[cV,line width=3.2pt] (10.239,0.620) -- (10.233,0.620);
\node[cV!75!black,font=\scriptsize,anchor=west] at (10.353,0.620) {$-0.0008$};
\node[anchor=east,font=\scriptsize] at (-0.15,0.000) {control: no disagreement \ \textcolor{black!45}{$n{=}8$}};
\draw[cGrey,line width=0.9pt] (0.634,0.000) -- (2.038,0.000);
\draw[cGrey,line width=0.9pt] (0.634,-0.090) -- (0.634,0.090);
\draw[cGrey,line width=0.9pt] (2.038,-0.090) -- (2.038,0.090);
\fill[cGrey] (1.332,0.000) circle (2.3pt);
\draw[cGrey,line width=3.2pt] (10.239,0.000) -- (10.200,0.000);
\node[cGrey!75!black,font=\scriptsize,anchor=west] at (10.320,0.000) {$-0.0051$};
\draw[cGrey,thin] (0.000,-0.450) -- (4.600,-0.450);
\draw[cGrey,thin] (0.159,-0.450) -- (0.159,-0.540);
\node[cGrey,font=\scriptsize,anchor=north] at (0.159,-0.570) {-2};
\draw[cGrey,thin] (0.952,-0.450) -- (0.952,-0.540);
\node[cGrey,font=\scriptsize,anchor=north] at (0.952,-0.570) {-1};
\draw[cGrey,thin] (1.745,-0.450) -- (1.745,-0.540);
\node[cGrey,font=\scriptsize,anchor=north] at (1.745,-0.570) {0};
\draw[cGrey,thin] (2.538,-0.450) -- (2.538,-0.540);
\node[cGrey,font=\scriptsize,anchor=north] at (2.538,-0.570) {1};
\draw[cGrey,thin] (3.331,-0.450) -- (3.331,-0.540);
\node[cGrey,font=\scriptsize,anchor=north] at (3.331,-0.570) {2};
\draw[cGrey,thin] (4.124,-0.450) -- (4.124,-0.540);
\node[cGrey,font=\scriptsize,anchor=north] at (4.124,-0.570) {3};
\node[cGrey,font=\scriptsize,anchor=north] at (2.300,-0.820) {$\Delta$ capability, in $0.0191$ floors};
\draw[cGrey,thin] (6.400,-0.450) -- (10.700,-0.450);
\draw[cGrey,thin] (7.168,-0.450) -- (7.168,-0.540);
\node[cGrey,font=\scriptsize,anchor=north] at (7.168,-0.570) {$-0.4$};
\draw[cGrey,thin] (7.936,-0.450) -- (7.936,-0.540);
\node[cGrey,font=\scriptsize,anchor=north] at (7.936,-0.570) {$-0.3$};
\draw[cGrey,thin] (8.704,-0.450) -- (8.704,-0.540);
\node[cGrey,font=\scriptsize,anchor=north] at (8.704,-0.570) {$-0.2$};
\draw[cGrey,thin] (9.471,-0.450) -- (9.471,-0.540);
\node[cGrey,font=\scriptsize,anchor=north] at (9.471,-0.570) {$-0.1$};
\draw[cGrey,thin] (10.239,-0.450) -- (10.239,-0.540);
\node[cGrey,font=\scriptsize,anchor=north] at (10.239,-0.570) {$0$};
\node[cGrey,font=\scriptsize,anchor=north] at (8.550,-0.820) {$\Delta$ allocation (share committed to one convention)};
\node[font=\small\bfseries,anchor=south] at (2.300,2.900) {what the model can do};
\node[font=\small\bfseries,anchor=south] at (8.550,2.900) {which convention it commits to};
\end{tikzpicture}}
\caption{\textbf{The switch, as a crossing.} The same five corpora on the two quantities
arrangement can move. \emph{Left:} the capability contrast in units of the $0.0191$ floor, with
its two-sided $95\%$ $t$ interval; the shaded band is $\pm1$ floor, inside which nothing is
resolvable. Only the coherent grid leaves it, all four other intervals containing zero.
\emph{Right:} which of two conventions the model commits to, where the pattern inverts. The
dashed line is the control's own allocation ($-0.0051$), the coherent grid and the keyed corpus
sit on it, and the two live conflicts run off to $-0.27$ and $-0.44$. Read together the panels
are the section's claim, that breaking coherence changes \emph{which} quantity is moved rather
than making an absent effect appear; neither says it alone.}
\label{fig:switch}
\end{figure}
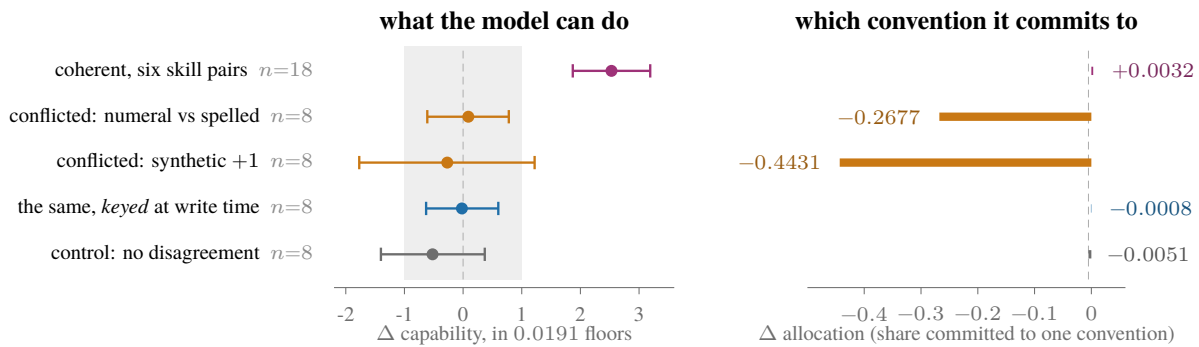

Three readings follow from the lower block, and Figure~\ref{fig:switch} draws them as one
crossing. \emph{Allocation starts.} The control moves the
share by $-0.0051$ at eight seeds, which is nothing; give the two halves incompatible
conventions and it moves $-0.2677$ and $-0.4431$, in a direction the control cannot produce,
since it is now the blocked arm that commits. \emph{Capability stays inside a floor throughout},
at $0.52$ floors in the control, $0.09$ and $0.27$ in the two conflicted corpora and $0.02$
under the key, with both conflicted intervals containing zero, $[-0.61,+0.78]$ floors on
numeral-against-spelled and $[-1.77,+1.22]$ on the synthetic corpus. That is a change from the
three-seed reading, where numeral-against-spelled gave $+0.0146$ on $[+0.01,+1.52]$ and excluded
zero by less than its own rounding; \S\ref{sec:license} reports what happened and why we said in
advance that it might. \emph{Both vanish together when the conflict is keyed}: writing the
convention into the input, so the halves no longer make incompatible demands on the same inputs,
returns allocation to $0.25\times$ the control and capability to $0.02$ floors on
$[-0.63,+0.60]$, the tightest row in the table. One intervention, on the data rather than on the
schedule, removes the whole phenomenon.

Forgetting cannot produce that pattern within one distribution. It has no term that switches on
when two sources disagree, none that a key at write time could disable, and no account of which
of two conventions a model commits to at fixed total accuracy. What both settings share is the
free-recipe limit's own account: arrangement acts only through what one stage failed to hold
after the next one ran, and what fails to be held differs, being content when the halves are
different skills and a convention when they are the same distribution under incompatible rules.
That last sentence is a reading rather than a controlled result, and \S\ref{sec:license} says
so; the controlled result is the lower block alone. \S\ref{sec:conflict} is where it is made to
survive its own controls: difficulty without disagreement, domain distance without disagreement,
three scales, two convention families, and the timescale at which the effect actually lives.

\textbf{Order is a different case: it does not survive its own selection.} Search over the six
matched pairs and keep the best, and the winner is algebra--combinatorics at $|\ASYM|=0.0619$,
which clears the $0.054$ this design resolves. Re-run the identical protocol twice more and that
cell reads $+0.0096$ and $-0.0146$, so the sign reverses and the magnitude falls by a factor of
four (\S\ref{sec:budget}). This is not an underpowered null but a \emph{winner's curse}
\citep{sculley2018winnerscurse,bastani2026winner}: the search returns a winner above the
resolution, and the winner is the selection. Composition at fixed volume, at $0.0101$, is
simply small.

So the three decisions are not one story. One has a systematic answer that the paper's own
headline contrast was blind to; one produces winners that do not replicate; one is small. Nor
does the first simply beat volume. The $+0.0483$ is a sum over two skills, about $+0.024$ each,
against volume's $+0.0463$ on a single skill, so volume still pays roughly twice as much per
skill on a coordinate where the comparison is fair. The reason to keep reading past that
arithmetic is not the size of the arrangement effect but what happens to it next.

\begin{table}[t]
\centering\footnotesize
\caption{\textbf{Every recipe decision is inside single-run noise once volume is fixed; volume
is not.} The right-hand block carries the claim: three contrasts at the learning budget
($3$ epochs, lr $3{\times}10^{-5}$, Qwen2.5-7B, $622$ rows per stage) against the $0.0191$ floor,
none resolvable. The last row is not a recipe decision and is the only quantity clearing the
floor. \emph{The left-hand block is context and not a controlled comparison, and an earlier
version of this table presented it as one:} the natural-size family is also a one-epoch budget
(step $39$ against the matched family's $117$), so size and budget move together, and each row
averages over a different set of pairs because the shuffled arm exists at natural sizes for only
two of the six. Every entry is computed rather than transcribed.}
\label{tab:matchedthree}
\begin{tabular}{llcccccl}
\toprule
 & contrast & \multicolumn{2}{c}{natural sizes, $1$ ep} & \multicolumn{2}{c}{matched $622$, $3$ ep} & $/$floor \\
\cmidrule(lr){3-4}\cmidrule(lr){5-6}
 & & mean $|\cdot|$ & pairs & mean $|\cdot|$ & pairs & \\
\midrule
composition & $\shuf$ against $\only_B$ & $0.0913$ & $2$ & $0.0101$ & $2$ & $0.53$ \\
order & $\Astep$ against $\text{B}\!\to\!\text{A}$ & $0.0347$ & $6$ & $0.0205$ & $4$ & $1.07$ \\
arrangement & $\shuf$ against $\blocked$ & $0.0396$ & $2$ & $0.0140$ & $6$ & $0.73$ \\
\midrule
\emph{volume} & $622$ rows against $311$, one skill & \multicolumn{2}{c}{---} & $0.0463$ & $4$ skills & $2.43$ \\
\bottomrule
\end{tabular}

\end{table}

What does clear the floor is the axis the limit never claimed was free. Doubling one skill's own
data ($311$ rows against $622$ of the same skill, four skills, three seeds, nothing else
varying) gains $0.0463$, $2.4\times$ the floor, all four skills in the same direction
(Table~\ref{tab:volume}), under a read-out rule frozen in the job script before the runs
existed. Three qualifications travel with it. The return is strongly heterogeneous (algebra and
calculus $\approx0.07$, geometry $0.017$, so data buys most where the model is already
learning); it is one doubling at one point on the curve, and nothing here prices a second; and
in long chain-of-thought reasoning SFT, repetition of existing data has been reported to beat
adding new data \citep{repetition2026}, so ``volume pays'' is measured here for \emph{new
same-skill rows at one doubling} and does not adjudicate new against repeated at other scales
or formats.

\begin{table}[t]
\centering\small
\caption{\textbf{Doubling one skill's own data, with nothing else varying} (Qwen2.5-7B, three
seeds; per-skill contrast floor $0.0191$; pre-registered read-out V1). }
\label{tab:volume}
\begin{tabular}{lrrrr}
\toprule
skill & $311$ rows & $622$ rows & gain & $/$floor \\
\midrule
algebra & $0.4875$ & $0.5556$ & $+0.0681$ & $3.56$ \\
calculus & $0.4438$ & $0.5146$ & $+0.0708$ & $3.71$ \\
geometry & $0.2833$ & $0.3007$ & $+0.0174$ & $0.91$ \\
combinatorics & $0.2646$ & $0.2937$ & $+0.0291$ & $1.52$ \\
\midrule
mean & & & $+0.0463$ & $2.43$ \\
\bottomrule
\end{tabular}

\end{table}

\paragraph{Both sides of the wall.} A wall claim must
measure the system quantity and the instrument quantity, and must report the cells that appear
to violate it. A mean of absolute contrasts is not a diameter, so the honest span of
$\mathcal{R}$ is the largest single-corpus contrast, $0.0619$. Against it stand two resolutions:
$\delta_{\text{seed}}=0.054$ for a three-seed design, and, for a practitioner comparing two
recipes with one run each, $\delta_{\text{run}}^{\text{med}}=0.031$ or
$\delta_{\text{run}}^{\text{sel}}=0.144$ depending on whether the cell was fixed in advance or
returned by a search (\S\ref{sec:budget}, Table~\ref{tab:runvar}). The grid's largest contrast
therefore sits \emph{between} the two, and it is precisely the cell that does not replicate:
re-run twice under an identical protocol, that $-0.0619$ reads $+0.0096$ and $-0.0146$. The
wall's prediction is not that a search finds nothing but that the search's best cell clears the
optimistic threshold, fails the realistic one, and does not survive re-running, which is what
the data shows in the one place a reader would look for a counterexample. It is also the one
derived bound in this paper shown to bite; the geometric ordering condition is violated in $36$
of $36$ measured pairs (\S\ref{sec:license}), and we report it as inoperative rather than
dressing it as support.

\paragraph{What this does and does not license.} It does not license ``arrangement effects do
not exist''. \S\ref{sec:conflict} shows the same instrument reading $|D|=0.014$ here reads
$0.23$ when coherence is broken, so the instrument is not blunt: the domain is coherent, and
what conflict does is enlarge $\mathcal{R}$ past the wall rather than sharpen the instrument.
Nor does it license ``stop measuring'', which the power analysis does not support. What it
licenses is a default and a priority. At corpus scales and budgets like these, an unscreened
blocked-versus-interleaved ablation inside one coherent domain will with high probability
measure its own noise; the measured lever is volume; and the screening question deciding whether
arrangement deserves runs at all is the one \S\ref{sec:conflict} isolates, \emph{would the two
corpora ever label the same input differently?}

\subsection{Breaking Coherence: Disagreement Is a Switch}\label{sec:conflict}

Equation~\eqref{eq:conflict} predicts that when two halves of a corpus make incompatible demands
on the same inputs, arrangement stops being free and blocked training wins ($D<0$), and the
prediction was pre-registered with its direction. What it is a prediction \emph{about} is
settled in \S\ref{sec:null}. This section is where the claim has to survive being something
other than a property of these particular corpora, and six controls decide it: difficulty
without disagreement, domain distance without disagreement, two convention families, three
model scales, a dose axis, and the timescale at which the effect lives. Every corpus below
splits one pool into two halves of matched size, holding questions, split, row counts and budget
identical by construction so that only the stated property differs, and four properties are
asserted before any GPU time is spent: user prompts byte-identical between control and conflict,
the convention shift exact on all $864$ pairs, the two halves disjoint, and the evaluation golds
carrying the same questions under the shifted convention. The one exception is the
domain-boundary row, which mixes two intact corpora because a domain boundary cannot be built by
splitting a pool (Table~\ref{tab:conflict}).

\begin{table}[t]
\centering\small
\caption{\textbf{Arrangement is switched on by disagreement, not by difficulty, domain distance,
one model or one convention family.} $D=\mathrm{acc}_B(\shuf)-\mathrm{acc}_B(\Astep)$ against
the $0.0191$ floor; \emph{guard} is $\mathrm{acc}_B(\only_B)$, which must clear the floor for a
row to be interpretable, and one dose point failed it and is excluded rather than read as zero.
Qwen2.5-7B, three seeds, learning budget, except where a scale or seed count is named. Source: the conflict result families. The synthetic $+1$ rows are the rebuilt instrument. Every row is a single execution of the protocol, so the rows are comparable to each other; the headline conflict row was additionally executed twice more from the same pool under different split seeds, giving $-0.1258$ and $-0.0825$ beside its $-0.0985$ (\S\ref{sec:budget}).}
\label{tab:conflict}
\begin{tabular}{llrrr}
\toprule
premise broken? & the two halves & $D$ & $/$floor & guard \\
\midrule
no, $8$ seeds & both halves numerals (control) & $-0.0086$ & $0.45$ & $0.3339$ \\
no & both halves \emph{spelled}: harder, no dispute & $+0.0012$ & $0.06$ & $0.1983$ \\
no & both halves Arabic (control) & $-0.0037$ & $0.20$ & $0.3121$ \\
no & algebra $+$ \emph{physics}: domain boundary, no dispute & $+0.0096$ & $0.50$ & $0.1642$ \\
\midrule
\textbf{yes}, $8$ seeds & numeral against spelled, \emph{both correct} & $\mathbf{-0.0985}$ & $\mathbf{5.16}$ & $0.2091$ \\
\textbf{yes} & Arabic against Roman, \emph{both correct} & $\mathbf{-0.1180}$ & $\mathbf{6.18}$ & $0.1532$ \\
\textbf{yes} & synthetic $+1$ on half the answers & $\mathbf{-0.2346}$ & $\mathbf{12.3}$ & --- \\
\midrule
\textbf{yes}, $3$B & numeral against spelled & $-0.0263$ & $1.38$ & $0.0459$ \\
no, $3$B & both halves numerals & $-0.0117$ & $0.61$ & $0.2725$ \\
\textbf{yes}, $14$B & numeral against spelled & $\mathbf{-0.2037}$ & $\mathbf{10.67}$ & $0.2579$ \\
no, $14$B & both halves numerals & $-0.0071$ & $0.37$ & $0.3983$ \\
\midrule
\textbf{yes}, dose $0.75$ & three quarters of one half disagree & $\mathbf{-0.0804}$ & $\mathbf{4.21}$ & $0.1554$ \\
\textbf{yes}, dose $0.50$ & half of one half disagrees & $\mathbf{-0.0425}$ & $\mathbf{2.23}$ & $0.0771$ \\
--- , dose $0.25$ & excluded: minority convention below its learnability edge & --- & --- & $0.0000$ \\
\bottomrule
\end{tabular}
\end{table}

Five statements, each carried by its own row against its own matched control.

\textbf{The switch fires on disagreement between two \emph{correct} conventions.} The same
problems with the boxed answer written as a numeral in one half and spelled out in the other:
$D=-0.0985$ at $5.16$ floors, all \emph{eight} seeds negative, control at $0.45$ floors, and
conflict minus control $-0.0720$ at $2.7\sigma$. Two further executions of the whole protocol,
built from the same pool under different split seeds and read out under a rule frozen beforehand,
give $-0.1258$ and $-0.0825$, so the three-execution mean is $-0.1023$ with replicate spread
$0.0219$ (\S\ref{sec:budget}). Nothing here is wrong by anyone's standard;
two conventions merely make incompatible demands on the same input span, which is exactly
$\kappa_A\neq\kappa_B$.

\textbf{Difficulty without disagreement does nothing.} Spelling \emph{both} halves gives the
harder convention throughout with nothing to dispute: $D=+0.0012$, $0.06\sigma$. The switch is
disagreement, not difficulty.

\textbf{Domain distance without disagreement does nothing.} Algebra co-trained with physics, a
mixture whose parts differ in everything except that each agrees with itself, passes its
learnability gate (physics-only $0.1642$ against zero-shot $0.1379$) and leaves arrangement
inert at $0.5\sigma$. \emph{Heterogeneity is not disagreement.} Mixing different things is free;
mixing contradicting things is not.

\textbf{It is not one convention family, one scale, or one dose.} Roman against Arabic numerals
replicates at $-0.1180$, $6.18$ floors on five seeds. Across scale the switch forms a ladder
moving with the guard, which is the ordering the mechanism predicts, since the more learnable
the minority convention the more there is to commit to: $1.38\sigma$ at $3$B with the convention
barely learnable ($0.0459$), $5.16$ floors at $7$B at guard $0.2091$, and $-0.2037$ at
$10.67\sigma$ at $14$B at guard $0.2579$, with the control there cleanly inert ($-0.0071$,
$0.37\sigma$, three seeds) and conflict minus control $-0.1966$ at $7.3\sigma$. The admitted
dose points also order: $-0.0425$ at dose $0.50$ ($2.23\sigma$, guard $0.0771$), $-0.0804$ at
$0.75$, $-0.0985$ at $1.0$, while the $0.25$ dose is \emph{excluded by the guard, not measured
as zero}. We decline to draw a dose--response curve through them, because the construction
confounds the disagreeing fraction with the minority convention's volume: the guard falls as the
dose falls, so part of any gradient is learnability rather than disagreement. That confounding
term is not an argument but a measurement, and it is a line. The guard reads $0.0000$, $0.0771$,
$0.1554$ and $0.2091$ at doses $0.25$, $0.50$, $0.75$ and $1.0$, which is
$0.282\,\mathrm{dose}-0.066$ with $R^2=0.993$ and residuals under $0.01$, crossing zero at dose
$0.234$. So the excluded row is not a separate phenomenon: it sits on the crossing, and
``below its learnability edge'' is a located number rather than a judgement. It is also why the
$D$ curve stays declined, since the term that would contaminate it runs linearly across the whole
range those three points span. The synthetic
$+1$ conflict, contradictory by construction, gives the ceiling at $-0.2346$, $12.3$ floors on
eight seeds.

\textbf{And it is not the evaluation siding with the last stage.} Every $D$ above scores the
convention the blocked arm wrote \emph{last}, and a referee is entitled to call that deck
stacked. Re-scored per problem at the \emph{better} of the two conventions, a rule with no
favourite, blocking's win vanishes where both conventions are correct at $7$B ($-0.0084$ and
$+0.0308$, both inside two floors, against the committed $-0.0985$ and $-0.1180$). \emph{Two
rows do not vanish, and we flag them here rather than in the appendix:} at $14$B and under the
synthetic conflict the agnostic contrast still fires, at $2.60$ and $4.65$ floors. Read
carelessly that says blocking bought capability. In both rows, though, the solve-sum stays low
($1.64$ and $0.27$ floors) while the best-convention-per-problem score drops, which is what
splitting $k{=}4$ samples across two conventions \emph{within one problem} produces and not what
solving fewer problems produces: reallocating samples between conventions conserves the sum and
lowers the maximum. For blocking to have bought capability the shuffled arm would have to solve
fewer problems, and it does not. Appendix~\ref{app:agnostic} tabulates all five rows and
Table~\ref{tab:agnostic} separates the hypotheses row by row.

\paragraph{Not the optimiser's memory: the switch lives at the stage timescale.} AdamW's
momentum is a low-pass filter on the gradient stream with time constant $1/(1-\beta_1)=10$
optimiser steps, so the contested directions could be averaging away \emph{inside the optimiser
state}, making the switch an artefact of exactly the stateful-optimiser kind
\citet{sweeney2026shuffle} warn about. We built the axis that decides this, with the
momentum-timescale hypothesis and its decision rule frozen before the runs: corpora of
alternating same-source blocks of exactly $L$ consecutive optimiser steps, $L$ from a single
step to one contiguous pass of a source per epoch, trainer shuffling off, everything else
identical. If optimiser memory were the mechanism the response turns on at $L^\star\approx10$.
It does not turn on there or anywhere near it. Through a $27$-fold range of block length,
spanning the momentum window and twice an independently measured residence time, every arm sits
at the interleaved arm's own allocation ($0.407$--$0.449$ against its $0.413$ on the synthetic
corpus; $0.40$--$0.44$ drifting $+0.031$ across a $45$-fold range on the natural one). The first
arm to move is one contiguous pass of a source per epoch ($0.492$), and full stage separation
sits far above ($0.869$), so the response turns on at the corpus scale and nowhere below it,
the opposite of what a $10$-step buffer predicts. The same sweep retires the batch-composition
reading: pure batches in \emph{random} order land with the shuffled arm, $3\%$ of the span. The
switch is invisible at every timescale a mini-batch or an optimiser buffer can see and fires
only when a source is contiguous at the \emph{stage} scale, which is what
Eq.~\eqref{eq:conflict} models.

\paragraph{What the switch moves: mostly allocation, not capability.} On the synthetic corpus
the two conventions are mutually exclusive, so $\mathrm{acc}_A+\mathrm{acc}_B$ measures what the
model solves at all, and across the twelve arms it spans $0.487$--$0.539$, a range of $2.7$
floors that is not zero but is a seventh of what the allocation does, while the $B$-share runs
$0.04\to0.41\to0.87$ from $A$-only through shuffled to blocked. On the spelled corpus at eight
seeds the effect is reallocation almost in full, the total moving $0.0016$ ($0.09$ floors) while
$\mathrm{acc}_B$ moves $0.0985$; on the Roman corpus at five seeds only about half is (total
$0.0465$, $2.62$ floors). The gap between the two cautions against reading either fraction too
precisely at these sample sizes, the spelled corpus's own having been four fifths before its
extension. What arrangement buys under conflict is therefore mostly \emph{which convention the
model commits to}, in a fraction depending on the conflict. This is the sober reading of the
$12.3$ floors and still the practically relevant one, since exact-match consistency of an output
convention is what production post-training cares about. \acompanion
\citep{dpd2026} takes the averaging account itself as its subject, and no claim here depends on
it. Relatedly, \citet{incompletelearning2026} isolates internal SFT-data inconsistency as a
distinct per-sample failure mode; our design shows the same inconsistency can be made to matter
or not \emph{by arrangement alone}, at fixed data.

\paragraph{How often does this happen without anyone arranging it?} Every corpus in
Table~\ref{tab:conflict} was made to disagree by us, and the objection that follows is the right
one. Its empirical half costs no training run: take public corpora practitioners actually pool,
match problems by text, and compare gold answers. Table~\ref{tab:wildsurvey} does that over
$164{,}984$ shared problems. Two golds can differ in two ways, and the distinction is this
paper's thesis rather than a detail. \emph{Convention} means they denote one value written two
ways ($\frac{3}{4}$ against $0.75$); \emph{value} means they denote different values, which is
label noise. A deliberately narrow canonicaliser decides, and what it cannot decide is reported
as undecided rather than promoted into either column.

The base rates say the screening question of \S\ref{sec:conclusion} has an answer depending on
\emph{what} is being pooled, and the ordering is not the obvious one. Between two independently
labelled public corpora, convention conflict is close to absent, $9$ instances in $92{,}359$
shared problems and $12$ in $7{,}169$; between two \emph{derivation pipelines} over the same
problems it is two hundred times more common, $1{,}284$ in $65{,}456$, because a tool-integrated
pipeline evaluates in Python and emits $0.75$ where a prose derivation writes $\frac{3}{4}$. But
that third row is not simply ``different pipelines disagree''. OpenR1 is also a pipeline, is also
paired against the tool-integrated corpus, and comes from a different organisation, yet its rate
is $31$ in $14{,}362$, nine times lower, the difference being that OpenR1's answers pass through
a canonicaliser (\texttt{math\_verify}) and NuminaMath-CoT's do not. What raises the rate is
therefore not that two pipelines produced the labels but that \emph{neither normalised the
answer's form}, a property a team can check of its own corpora before mixing them and a cheaper
intervention than anything in \S\ref{sec:conclusion}'s decision procedure. The last row changes
the reading once more: inside a \emph{single} public corpus, on problems it happens to list
twice, value disagreement outnumbers convention disagreement roughly twenty to one. The dominant
pathology of real mixtures is noise rather than contested convention, which bounds how much of
production data cleaning this section speaks to.

\begin{table}[t]
\centering\small
\caption{\textbf{How often two public corpora disagree about the same problem, and in which of
the two ways.} Problems are matched on normalised text; \emph{convention} means the two golds
denote the same value written differently, \emph{value} means they do not, and \emph{undecided}
is what the canonicaliser declined to rule on, reported rather than folded into either column.
Convention conflict is near-absent between independent labellings and two orders of magnitude
more common between two derivation pipelines.}
\label{tab:wildsurvey}
\setlength{\tabcolsep}{3pt}\footnotesize
\begin{tabular}{llrrrrr}
\toprule
corpora & relation & shared & same & \textbf{convention} & value & undec. \\
\midrule
NuminaMath-CoT vs OpenR1-220k & two independent labellings & $92{,}359$ & $90{,}599$ & $\mathbf{9}$ & $51$ & $1{,}700$ \\
NuminaMath-CoT vs GSM8K & two independent labellings & $7{,}169$ & $4{,}928$ & $\mathbf{12}$ & $148$ & $2{,}081$ \\
OpenR1-220k vs NuminaMath-TIR & two pipelines, two organisations & $14{,}362$ & $13{,}641$ & $\mathbf{31}$ & $70$ & $620$ \\
NuminaMath-CoT vs NuminaMath-TIR & two pipelines, one curator & $65{,}456$ & $52{,}626$ & $\mathbf{1{,}284}$ & $396$ & $11{,}150$ \\
\midrule
NuminaMath-CoT with itself & duplicate problems & $42{,}535$ & $33{,}812$ & $\mathbf{161}$ & $3{,}153$ & $5{,}409$ \\
\bottomrule
\end{tabular}

\end{table}

\paragraph{The missing cells.} Three gaps bound what Table~\ref{tab:conflict} licenses,
and we state them at the same volume as the firing rows.

\emph{Every firing conflict was built, and the attempt to find one in the wild returned an
instrument result rather than a row.} The corpus existed and \emph{we did not build the
disagreement}: the $1{,}137$ problems of Table~\ref{tab:wildsurvey}'s third row, in two disjoint
halves of $443$, with the conflict arm untouched public data and the \emph{control} carrying our
edit, so style differs identically in both arms and only convention differs between them. Forty
arms later the registered verdict is \textbf{W3}: the control
moves at $4.45$ floors, so the frozen rule permits no read-out. The mechanism is not the one W3
anticipated and was found by reading the arm table rather than the verdict. Every arm of both
conditions sits at allocation share $0.500$ to $0.506$ with
$\mathrm{acc}_A\approx\mathrm{acc}_B$ inside every cell, because the scorer's first branch tests
\emph{mathematical} equality and this corpus is by definition a set of pairs whose golds denote
one value written two ways: \textbf{$250$ of $250$ evaluation gold pairs are equal under our own
scorer}. The build-time preflight did assert
that the golds differ, as \emph{strings}, and passed on all $250$. One assertion, the wrong
operator, forty trainings downstream.

\emph{The learnability guard could not have caught it, and that is the transferable part.} It
read $0.3912$, twenty times the floor, because $\mathrm{acc}_B$ is high when the model solves the
problems rather than when it adopts a minority convention. Every conflict guard we know of, ours
included, asks whether the minority content was learned and computes that with the same scorer
the effect will be measured with; the wild row is what that circularity looks like when it
closes. \emph{A guard computed with the scorer cannot certify that the scorer can see the axis
the guard is about.} The failure is contained, checked rather than assumed: numeral against
spelled, Arabic against Roman and the synthetic $+1$ give $0/292$, $0/276$ and $0/342$
mathematically equal pairs, and only the wild corpus fails. The repair costs no training and
belongs in front of any experiment on contested data rather than in an erratum:

\begin{quote}
\textbf{Separability precondition.} Train on convention $A$ alone and score that single model on
\emph{both} golds. Require $\mathrm{acc}_A/\mathrm{acc}_B \ge 2$. Only then read any mixed arm.
\end{quote}

\noindent Applied to every row this paper publishes it passes at $16.6\times$ to infinite, and the
same-convention controls fail it as they must, which is the sign that it is measuring the intended
thing. It is now a preflight in this project rather than a read-out, and we recommend it wherever
a metric might canonicalise the axis under study, which includes any exact-match evaluation of
mathematics, code or structured output.

\emph{So the cell stays open with a sharper specification} (\S\ref{sec:conclusion}): a wild
conflict corpus must pair conventions that are not numerically equal, which excludes
exact-form-against-decimal, $1{,}284$ of the $1{,}284$ real pipeline disagreements we found. The
commonest real-world convention conflict available to us is one an exact-match metric cannot
score, which is a finding about evaluation rather than about arrangement.

\emph{Read across \ourline\ it is not only a cost.} \thecompanion
\citep{er2026} states twice that the wall it derives is escaped \emph{only} by writing the
convention into the input. This run is a measured counterexample: maximal convention entropy
given the input, no key anywhere, and no allocation cost at all, because the scorer accepts both
forms. Canonicalising the metric is a third escape, and the premise doing the work sits inside
that paper's proof rather than in its escape list, namely that the conventions be mutually
exclusive under the metric. It carries the escape now, with the qualification that matters: it
removes the \emph{penalty} and not the disagreement, so it is available exactly when the
conventions are interchangeable for whoever consumes the output. Numeric form is such a case; the
conventions in Table~\ref{tab:conflict} are not, which is why they still fire. The
production-mixture reading remains a hypothesis with a measured base rate behind it, not a result
this paper contains.

\emph{The firing rows are no longer one pretraining family, and the replication costs the ladder
its generality.} \S\ref{sec:capacity} shows sign-level conclusions reversing between Qwen2.5 and
Qwen3 on these very skills, so by this paper's own standard a one-family switch would not
extrapolate, and the replication was registered against that objection before it ran
(the registration). On Qwen3-8B base, five seeds, same corpus and
budget, separability checked first and passing at infinity: guard $0.0520$, $D=-0.0875$ at
$\mathbf{4.58}$ floors on $[-0.1038,-0.0712]$, all five seeds negative, control $+0.0178$ ($0.93$
floors), conflict minus control $-0.1053$. That is the registered \textbf{S1} branch, and the
switch now spans four scales in two families.

\emph{And the row costs us something we published.} Amendment A4, frozen before any
\texttt{A\_then\_B} file existed, recorded that Qwen3-8B's guard sits at Qwen2.5-3B's rung rather
than $7$B's. Our own ladder has switch magnitude tracking the guard, on the reading that the more
learnable the minority convention the more there is to commit to, so it predicted a marginal
result here, and A4 bound us to report the outcome either way without upgrading the ladder from
one point. The prediction failed in the direction that costs us:

\begin{center}\small
\begin{tabular}{lccc}
\toprule
model & guard $\mathrm{acc}_B(\texttt{B\_only})$ & $D$ & floors \\
\midrule
Qwen2.5-$3$B & $0.0459$ & $-0.0263$ & $1.38$ \\
\textbf{Qwen3-$8$B} & $\mathbf{0.0520}$ & $\mathbf{-0.0875}$ & $\mathbf{4.58}$ \\
Qwen2.5-$7$B & $0.2091$ & $-0.0985$ & $5.16$ \\
Qwen2.5-$14$B & $0.2579$ & $-0.2037$ & $10.67$ \\
\bottomrule
\end{tabular}
\end{center}

\noindent At the $3$B rung's loading the effect is $3.3\times$ larger, and nearly the size of the
$7$B row at a quarter of its guard. The ladder is monotone \emph{within} Qwen2.5 and this point is
a counterexample to reading it as a law across families; guard is not the only axis, and we do not
name the second one from a single point. Two statements survive: within a family, magnitude and
guard move together, and the switch does not need a large guard to appear. Llama-3 sits below the
power gate.

\emph{The boundary control has $n{=}1$.} The firing rows carry two convention families and two
scales, while the equally load-bearing ``heterogeneity is not disagreement'' null rests on one
skill pair at one scale. A second cross-domain pair owes it the same replication standard, and
until then that control is one clean instance, not a replicated result.

\section{What Recipe Search Cannot Buy: Patience and Room}\label{sec:budget}

Equation~\eqref{eq:expansion} says an order effect is a competition between a symmetric $O(T)$
term and an antisymmetric $O(T^2)$ term, so its sign must move along the budget axis. We
measured the curve directly: checkpoints every $13$ optimiser steps through both directions of
algebra--combinatorics, three seeds, $\ASYM$ evaluated at each (Figure~\ref{fig:traj}). The
effect crosses zero \emph{three times inside a single run}, and the three are not equal
evidence. The first two bracket excursions of $+0.112$ and $-0.116$, about $6\sigma$ and larger
than any endpoint value we measured, while the third joins $-0.013$ to $+0.010$ with both
endpoints inside the $\pm0.0191$ band, crossing the axis but not the noise. Past step $65$ the
curve decays into that band and stays there, so the honest count is two evidential sign changes
plus a late wander. That is still more than one, which is the whole of the argument: a single
endpoint cannot represent this curve. All four pairs measured this way change sign within a run,
and over the last four checkpoints none holds a stable sign.

\begin{figure}[t]
\centering
\begin{tikzpicture}[font=\small,>=Stealth]
\draw[cGrey,thin] (0,0.20) -- (0,3.60);
\foreach \v in {-0.10,-0.05,0.05,0.10}
 {\draw[cGrey,thin] (0,{1.9+\v*13}) -- (-0.08,{1.9+\v*13});
 \node[cGrey,font=\scriptsize,anchor=east] at (-0.12,{1.9+\v*13}) {$\v$};}
\node[cGrey,font=\scriptsize,anchor=east] at (-0.12,1.9) {$0$};
\node[cGrey,font=\scriptsize,rotate=90,anchor=south] at (-1.05,1.9) {$\ASYM$};
\draw[cGrey,densely dashed] (0.000,1.900) -- (9.828,1.900);
\fill[cGrey,opacity=0.18] (0.000,1.652) rectangle (9.828,2.148);
\draw[cA,line width=1.2pt] (1.014,1.204) -- (2.028,3.353) -- (3.042,2.631) -- (4.056,0.392) -- (5.070,1.692) -- (6.084,1.827) -- (7.098,1.674) -- (8.112,1.728) -- (9.126,2.025);
\fill[cA] (1.014,1.204) circle (2.0pt);
\fill[cA] (2.028,3.353) circle (2.0pt);
\fill[cA] (3.042,2.631) circle (2.0pt);
\fill[cA] (4.056,0.392) circle (2.0pt);
\fill[cGrey] (5.070,1.692) circle (2.0pt);
\fill[cGrey] (6.084,1.827) circle (2.0pt);
\fill[cGrey] (7.098,1.674) circle (2.0pt);
\fill[cGrey] (8.112,1.728) circle (2.0pt);
\fill[cGrey] (9.126,2.025) circle (2.0pt);
\draw[cNO,thin,densely dotted] (1.342,0.180) -- (1.342,3.620);
\draw[cNO,thin,densely dotted] (3.370,0.180) -- (3.370,3.620);
\draw[cNO,thin,densely dotted] (8.697,0.180) -- (8.697,3.620);
\node[cGrey,font=\scriptsize,anchor=north] at (1.014,0.100) {13};
\node[cGrey,font=\scriptsize,anchor=north] at (3.042,0.100) {39};
\node[cGrey,font=\scriptsize,anchor=north] at (5.070,0.100) {65};
\node[cGrey,font=\scriptsize,anchor=north] at (7.098,0.100) {91};
\node[cGrey,font=\scriptsize,anchor=north] at (9.126,0.100) {117};
\node[cGrey,font=\scriptsize,anchor=north] at (4.9,-0.30) {optimiser step (checkpoint every $13$)};
\node[cNO,font=\scriptsize,anchor=west] at (10.1,3.30) {dotted: zero crossings};
\node[cGrey,font=\scriptsize,anchor=west,text width=3.9cm] at (10.1,2.70)
 {shaded: $\pm1$ s.e.\ of $\ASYM$ ($0.0191$); grey markers lie inside it};
\node[cA!85!black,font=\scriptsize,anchor=west,text width=3.9cm] at (10.1,1.05)
 {an endpoint study reports one point on a curve like this one, which has already crossed
 zero three times before its own last point};
\end{tikzpicture}
\caption{\textbf{The order effect is a function of the budget, not a scalar.} $\ASYM$ for
algebra--combinatorics at every $13$th optimiser step of both stage-2 directions (Qwen2.5-7B,
$622$ rows per stage, $3$ epochs, lr $3{\times}10^{-5}$; the axis spans the entire stage-2 run),
with zero crossings interpolated at steps $17$, $43$ and $112$. This pair is an example, not the
evidence base: all four pairs measured this way change sign inside a single run, and in every
pair at least one crossing is made by each seed separately (Table~\ref{tab:seedcross}), here
those at $17$ and $43$, both $3/3$, while the in-band crossing at $112$ is $1/3$.}
\label{fig:traj}
\end{figure}
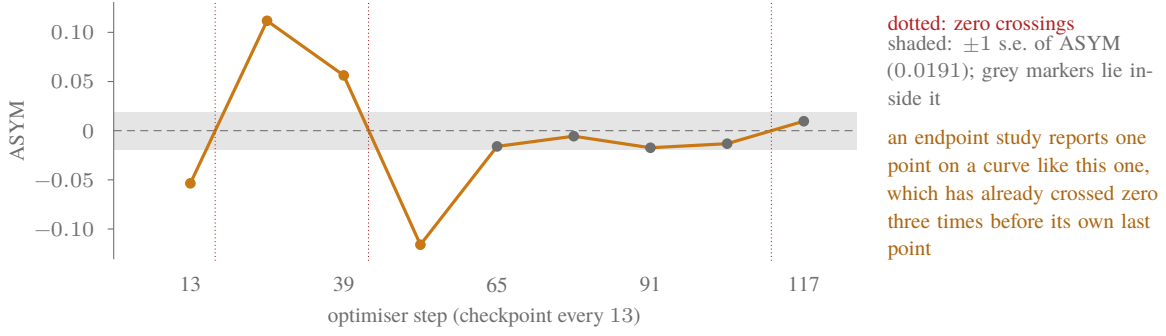

The endpoint version of the same fact: re-measuring all six volume-matched pairs at a budget
where training actually teaches (three epochs, lr $3{\times}10^{-5}$; at the original one-epoch
budget three of four skills end \emph{below} their zero-shot baseline, so contrasts there
compare two ways of damaging the model) moves every pair, four reversing sign and two falling
into noise (Table~\ref{tab:budgetflip}). The literature's inconsistency on curricula is what
this curve looks like when sampled once per paper: a study at one endpoint can find the effect
present, absent, or reversed without contradicting another
\citep{goodcurriculum2025,beyondrandom2025,interleaved2025}. \emph{An order effect reported
without its budget names one point on a curve.} Concurrent work finds the same shape at a
different grain, with order sensitivity existing on over a third of (model, task-pair)
combinations while its onset scale resists prediction \citep{pairwisefragile2026}, and influence
itself is reported to flip sign during training
\citep{pruthi2020tracin,influencedynamics2025}.

\begin{table}[t]
\centering\small
\caption{\textbf{The directional order effect at two budgets}, on the same skills, the same
$622$ rows per stage and the same seeds. Standard errors use the project floor. The only pair
significant at the learning budget is significant in the \emph{opposite} direction from the
published budget, and that endpoint carries a replication caveat (see text).}
\label{tab:budgetflip}
\begin{tabular}{lccccl}
\toprule
skill pair & \multicolumn{2}{c}{$1$ ep, lr $10^{-5}$} & \multicolumn{2}{c}{$3$ ep, lr $3{\times}10^{-5}$} & outcome \\
\cmidrule(lr){2-3}\cmidrule(lr){4-5}
 & $\ASYM$ & $|{\cdot}|/$s.e. & $\ASYM$ & $|{\cdot}|/$s.e. & \\
\midrule
algebra--calculus & $+0.0743$ & 3.9 & $+0.0020$ & 0.1 & falls into noise \\
algebra--geometry & $+0.0320$ & 1.7 & $-0.0333$ & 1.7 & \textbf{sign reversal} \\
algebra--combinatorics & $+0.0779$ & 4.1 & $-0.0619$ & 3.2 & \textbf{sign reversal} \\
calculus--geometry & $-0.0396$ & 2.1 & $+0.0041$ & 0.2 & \textbf{sign reversal} \\
calculus--combinatorics & $+0.0104$ & 0.5 & $-0.0125$ & 0.7 & \textbf{sign reversal} \\
geometry--combinatorics & $+0.0506$ & 2.7 & $+0.0091$ & 0.5 & falls into noise \\
\bottomrule
\end{tabular}

\end{table}

\paragraph{The price of one stable endpoint.} Because no arm is
normally trained twice, we measured what re-running costs. Three independent runs of the same
protocol on algebra--combinatorics at the learning budget give means $-0.0619$, $+0.0096$,
$-0.0146$: run-level standard deviation $0.0364$, of which seed noise explains about a quarter
of the variance. Stabilising that one endpoint to within a floor costs $\approx4$ independent
runs, twelve trainings, and at three seeds each that is $36$. At the low budget the same pair
\emph{does} replicate ($+0.0667$ against $+0.0779$, $0.6$ s.e.). Order measurements near the
transient are expensive by nature, and error bars pricing only seed noise understate them.

That constant, however, was estimated on the one pair re-run \emph{because} it failed to
replicate, and selecting a cell for having moved and then measuring how much cells move is a
biased estimator whose bias runs upward, in the direction that makes this paper's only biting
bound bite harder. The three pairs with no replicate were therefore run under a read-out frozen
beforehand, so all six now carry at least two
independent executions of the identical protocol (Table~\ref{tab:runvar}). The registered
outcome that fired sharpens the wall rather than softening it: run-level spread is a property of
the cell, not a constant of the protocol. The selected pair has the largest spread of the six
and the median of the other five is $0.0079$, a factor of $4.6$ smaller, so the \emph{median
cell} resolves $\delta=0.031$ while the \emph{selected cell} resolves $0.144$. A practitioner
running a search does not face the average cell's noise; they face the noise of the cell their
search selected, and selection lands preferentially on cells near a transient, which are exactly
the noisy ones. That is Proposition~\ref{prop:wall}'s winner's curse restated in variance rather
than in means, and it was invisible while the constant rested on one pair. Two of the six pairs
\emph{reverse sign} between independent runs, so the signature Proposition~\ref{prop:wall} names
is now observed on a second pair that no one selected. The number belonging in a recipe-search
decision is $\delta_{\text{run}}^{\text{sel}}=0.144$.

\paragraph{The same standard, turned on our own headline.} A paper that prices run-level noise
and then reports its own positive result at seed-level replication is claiming a licence it has
just denied everyone else. The switch was therefore re-executed too, and in the stricter sense:
\S\ref{sec:conflict}'s protocol fixes the pool, the halving, the conventions, the budget, the arms
and the evaluation, and leaves exactly one thing free, the split seed deciding which rows carry
which convention. Rebuilding under two further split seeds gave corpora overlapping the published
halves by $51.7\%$ and $49.2\%$, and $48$ trainings against a read-out frozen before any of them
started, with the demotion branch committed in advance. Guard and separability passed on both
replicates. The replicate means are $-0.0985$ (published, eight seeds), $-0.1258$ and $-0.0825$,
all six new seeds negative, both new controls inert at $0.26$ and $0.04$ floors, and the
conversion that turns $0.0364$ into $0.144$ turns the replicate spread into
$\delta_{\text{run}}^{\text{switch}}=0.087$.

The two constants belong side by side, because together they are this paper's argument in one
line. What a search selects is a cell whose diameter is $0.0619$ and whose re-execution noise is
$0.144$. What disagreement produces is a contrast of $0.10$ whose re-execution noise is $0.087$.
The wall is not that recipe effects are small; it is that their ratio to the resolution is wrong,
and the one departure with a systematic answer clears it on both terms at once.

\begin{table}[t]
\centering\small
\caption{\textbf{Run-to-run variation on all six pairs, at the learning budget.} Every
independent execution of the identical protocol that exists; $\delta$ converts each spread by
the project's MDE convention ($2.8\,s\sqrt{2}$, the same one that turns $0.0364$ into $0.144$).
The pair in bold is the one $\delta_{\text{run}}$ was originally estimated on; it was re-run
because it had failed to replicate, and it has the largest spread of the six. Ordering is by
spread, which is also the reading. Read out under a criterion frozen before job $2827$ ran.}
\label{tab:runvar}
\begin{tabular}{lccccc}
\toprule
skill pair & runs & $\ASYM$ per run & s.d. & $\delta$ & sign \\
\midrule
calculus--geometry & 2 & $+0.0041$, $+0.0035$ & $0.0004$ & $0.002$ & stable \\
algebra--calculus & 2 & $+0.0020$, $+0.0055$ & $0.0025$ & $0.010$ & stable \\
calculus--combinatorics & 2 & $-0.0125$, $-0.0014$ & $0.0078$ & $0.031$ & stable \\
algebra--geometry & 2 & $-0.0333$, $-0.0208$ & $0.0088$ & $0.035$ & stable \\
geometry--combinatorics & 2 & $+0.0091$, $-0.0083$ & $0.0123$ & $0.049$ & \textbf{flips} \\
\textbf{algebra--combinatorics} & 3 & $-0.0619$, $+0.0096$, $-0.0146$ & $\mathbf{0.0364}$ & $\mathbf{0.144}$ & \textbf{flips} \\
\bottomrule
\end{tabular}

\end{table}

\paragraph{The optimiser is not memoryless.} Everything
above was trained with AdamW, while Eq.~\eqref{eq:expansion} is derived for memoryless gradient
descent, and the gap is not pedantic. \citet{sweeney2026shuffle} shows that for a memoryless
optimiser, reordering an equal multiset has no first-order endpoint term, the leading local
contrast being exactly the $O(\eta^2)$ gradient bracket Eq.~\eqref{eq:expansion} carries, while
a stateful fixed-clock optimiser adds an $O(\eta)$ order-dependent \emph{noise} channel through
its moment buffers (measured order-variance slopes $1.83$ for AdamW against $4.00$ for SGD),
large enough for a single seed to flip a close comparison. An alternative reading of
Figure~\ref{fig:traj} therefore exists and must be faced: the crossings could be optimiser
artefact rather than the two-term competition. Three measurements bound it, and one control is
missing.

\emph{The channel is priced into our floor, and what it cannot produce we measured rather than
asserted.} $\hat\sigma$ is estimated across seeds differing in data order (the interleaved
corpus is re-drawn per seed), so first-order shuffle noise is inside the $0.0191$. More to the
point, the documented channel is a \emph{variance} channel, independent across seeds by
construction: it can move any one seed's curve but cannot put the same crossing in all of them.
So we asked of each seed's own trajectory, with no pooling anywhere in the test, whether it
changes sign in the same interval as its siblings (Table~\ref{tab:seedcross}). Six crossings are
unanimous and every one of the four pairs has at least one, while the remaining ten pooled
crossings are exactly those whose flanking checkpoints sit inside the noise band, which is where
a variance channel should and does show up. The published trajectory splits the same way, its
band-clearing crossings at steps $17$ and $43$ being $3/3$ and its in-band crossing at $112$
$1/3$. We note the temptation we avoided, of asking instead whether seeds agree in sign where
the \emph{pooled mean} clears the floor; that ratio is $18/22$, but the question is circular,
since a pooled mean only clears the floor when the seeds already agree.

\emph{The channel is visible exactly where it should be.} The run-to-run excess ($0.0364$
against $0.0191$) is the residue such a channel should leave, so we read
\citeauthor{sweeney2026shuffle}'s result as the likely mechanism behind our own replication
caveat: our four-runs-per-endpoint price is their single-seed warning converted into a budget.
\emph{And the memoryless mechanism is sufficient without optimiser state.} Plain gradient descent
reproduces the crossings in-model ($72\%$ of random non-commuting draws,
Figure~\ref{fig:transient}); both stage-2 arms of $\ASYM$ begin from freshly initialised
optimiser state, so the stage boundary's clock is symmetric across the contrast; and the
second-moment timescale $1/(1-\beta_2)=1000$ steps exceeds every stage-2 run here, leaving the
$10$-step momentum window as the only live memory.

The boundary is therefore narrow. \textbf{The existence of the transient is robust}: the sign
changes inside a single run, in all four pairs, with each pair's band-clearing crossing
reproduced by every seed on its own. \textbf{Its attribution to Eq.~\eqref{eq:expansion} is the
best-supported mechanism rather than an exclusive one}, and what is not excluded is specific: a
\emph{systematic} AdamW contribution, a bias rather than a variance, to \emph{where} the
crossings sit. Since the paper's claim is existence and budget-dependence rather than location,
that gap carries no conclusion, and nowhere do we recommend a budget because a crossing was
measured at a particular step. The experiment that would close it was pre-registered before any
of this was written, one pair's trajectory replicated under plain SGD with a guard reporting
\textsc{not measurable} if no calibrated SGD arm reaches the AdamW endpoints
(the registration). It has not been run, and the registration stands
rather than being withdrawn quietly, because the one documented alternative mechanism is the
variance channel, the unanimity test rules it out on our own data, and the residual possibility
is a bias channel no published account proposes and that would move only a quantity this paper
does not use.

\begin{table}[t]
\centering\small
\caption{\textbf{The crossings that clear the noise band are made by every seed separately; the
ones inside it are not.} Each seed's own $\ASYM(t)$ is tested for a sign change in each interval
between adjacent checkpoints, with no pooling anywhere in the column, so the test is not the
circular one of asking whether seeds agree where their mean is large. A crossing is
\emph{unanimous} when every seed of that pair changes sign in the same interval, which a
seed-independent variance channel cannot satisfy. Flanking values are the pooled $|\ASYM|$ at
the two bracketing checkpoints, in $0.0191$ floors. Qwen2.5-7B, learning budget; the
combinatorics--geometry pair has two seeds, the rest three. }
\label{tab:seedcross}
\begin{tabular}{lccll}
\toprule
pair & crossing at & flanking $|\ASYM|$, floors & \multicolumn{2}{c}{seeds crossing in that interval} \\
\midrule
algebra-combinatorics (3 seeds) & step $17$ & $2.8$, $5.9$ & $\mathbf{3/3}$ & \textbf{unanimous} \\
 & step $43$ & $2.9$, $6.1$ & $\mathbf{3/3}$ & \textbf{unanimous} \\
 & step $112$ & $0.7$, $0.5$ & $1/3$ &  \\
\midrule
algebra-calculus (3 seeds) & step $45$ & $1.9$, $2.2$ & $\mathbf{3/3}$ & \textbf{unanimous} \\
 & step $76$ & $3.4$, $0.8$ & $2/3$ &  \\
 & step $85$ & $0.8$, $0.7$ & $1/3$ &  \\
 & step $116$ & $3.4$, $0.3$ & $1/3$ &  \\
\midrule
calculus-geometry (3 seeds) & step $34$ & $6.1$, $4.2$ & $\mathbf{3/3}$ & \textbf{unanimous} \\
 & step $51$ & $4.2$, $0.4$ & $1/3$ &  \\
 & step $54$ & $0.4$, $1.9$ & $0/3$ &  \\
 & step $84$ & $1.0$, $1.0$ & $0/3$ &  \\
 & step $115$ & $1.1$, $0.2$ & $1/3$ &  \\
\midrule
combinatorics-geometry (2 seeds) & step $34$ & $3.8$, $2.6$ & $\mathbf{2/2}$ & \textbf{unanimous} \\
 & step $74$ & $2.0$, $0.9$ & $1/2$ &  \\
 & step $103$ & $1.4$, $0.1$ & $\mathbf{2/2}$ & \textbf{unanimous} \\
 & step $105$ & $0.1$, $2.1$ & $1/2$ &  \\
\bottomrule
\end{tabular}

\end{table}

\paragraph{A per-skill structure, reported as exploratory.} At one budget the six matched pairwise
effects are consistent with a per-skill potential, $\ASYM(X,Y)\approx\phi_X-\phi_Y$, which if real
would replace $n!$ order trials with $n-1$ measurements. We do not claim it. It is three free
parameters against six observations; four of those six pairs reverse sign at the other budget we
measured, so the quantity being fitted is itself budget-dependent; its out-of-sample ranking rests
on three triples ($\rho=0.657$, $n=3$); and its residual beats the marginal floor only because
same-seed arms share checkpoints. Appendix~\ref{app:potential} reports it in full and
Appendix~\ref{app:census} counts it, because it would be the cheapest result in this paper if it
held. Nothing else here depends on it, and the $\ge6$-skill out-of-sample test that would make it
a claim is registered with a withdrawal clause and unrun.

\subsection{Breaking Room: Motivated, Not Demonstrated}\label{sec:capacity}

Theorem~\ref{thm:limit}c makes observed negative co-training value at matched budget evidence of
a broken premise, and a capacity constraint is the natural candidate, so we imposed one
directly, varying only LoRA rank $r\in\{4,\dots,128\}$ at fixed model and data. The predicted
sign change of $V(r)$ appeared, monotone, in $11/12$ (pair, rank) cells across three seeds, and
did not survive its own controls. Decomposing which arm moves shows two thirds of the change in
$V$ is the \emph{single-skill baseline collapsing}, the arm with half the data, which is the
signature of overfitting rather than capacity; a rank$\times$learning-rate grid shows the
learning-rate axis moves $V$ about as much as rank does; and the volume-matched re-run (both
arms $622$ rows at every rank) removes the crossing entirely for the pair that carried it,
leaving the other pair one seed-consistent cell at $1.1\sigma$, below the $\mathrm{MDE}_{80}$
gating every other cell (Table~\ref{tab:capmatched}). The theory's own prediction for the
crossing's \emph{location} was excluded with the direction reversed (ratio $0.382$, $95\%$ CI
$[0.231,0.583]$ against predicted $1.0$--$1.27$). We report the sequence because the arithmetic
that caught it is reusable.

\begin{figure}[t]
\centering
\begin{tikzpicture}[font=\small,>=Stealth]
\draw[cGrey!60,densely dashed] (0.000,1.875) -- (6.400,1.875);
\draw[cGrey,thin] (-0.300,0.000) -- (6.700,0.000);
\draw[cGrey,thin] (-0.300,0.000) -- (-0.300,3.000);
\draw[cGrey,thin] (0.000,0.000) -- (0.000,-0.090);
\node[cGrey,font=\scriptsize,anchor=north] at (0.000,-0.120) {4};
\draw[cGrey,thin] (1.280,0.000) -- (1.280,-0.090);
\node[cGrey,font=\scriptsize,anchor=north] at (1.280,-0.120) {8};
\draw[cGrey,thin] (2.560,0.000) -- (2.560,-0.090);
\node[cGrey,font=\scriptsize,anchor=north] at (2.560,-0.120) {16};
\draw[cGrey,thin] (3.840,0.000) -- (3.840,-0.090);
\node[cGrey,font=\scriptsize,anchor=north] at (3.840,-0.120) {32};
\draw[cGrey,thin] (5.120,0.000) -- (5.120,-0.090);
\node[cGrey,font=\scriptsize,anchor=north] at (5.120,-0.120) {64};
\draw[cGrey,thin] (6.400,0.000) -- (6.400,-0.090);
\node[cGrey,font=\scriptsize,anchor=north] at (6.400,-0.120) {128};
\draw[cGrey,thin] (-0.300,0.000) -- (-0.220,0.000);
\node[cGrey,font=\scriptsize,anchor=east] at (-0.360,0.000) {$-0.2$};
\draw[cGrey,thin] (-0.300,0.469) -- (-0.220,0.469);
\node[cGrey,font=\scriptsize,anchor=east] at (-0.360,0.469) {$-0.15$};
\draw[cGrey,thin] (-0.300,0.938) -- (-0.220,0.938);
\node[cGrey,font=\scriptsize,anchor=east] at (-0.360,0.938) {$-0.1$};
\draw[cGrey,thin] (-0.300,1.406) -- (-0.220,1.406);
\node[cGrey,font=\scriptsize,anchor=east] at (-0.360,1.406) {$-0.05$};
\draw[cGrey,thin] (-0.300,1.875) -- (-0.220,1.875);
\node[cGrey,font=\scriptsize,anchor=east] at (-0.360,1.875) {$0$};
\draw[cGrey,thin] (-0.300,2.344) -- (-0.220,2.344);
\node[cGrey,font=\scriptsize,anchor=east] at (-0.360,2.344) {$0.05$};
\draw[cGrey,thin] (-0.300,2.813) -- (-0.220,2.813);
\node[cGrey,font=\scriptsize,anchor=east] at (-0.360,2.813) {$0.1$};
\node[cGrey,font=\scriptsize,anchor=north] at (3.200,-0.480) {LoRA rank $r$: the capacity knob ($\log_2$ spacing)};
\node[cGrey,font=\scriptsize,rotate=90,anchor=south] at (-1.050,1.500) {co-training value $V$};
\draw[cA,solid,line width=1.2pt] (0.000,0.332) -- (1.280,0.683) -- (2.560,1.959) -- (3.840,2.435) -- (5.120,2.727) -- (6.400,2.624);
\fill[cA] (0.000,0.332) circle (1.9pt);
\fill[cA] (1.280,0.683) circle (1.9pt);
\fill[cA] (2.560,1.959) circle (1.9pt);
\fill[cA] (3.840,2.435) circle (1.9pt);
\fill[cA] (5.120,2.727) circle (1.9pt);
\fill[cA] (6.400,2.624) circle (1.9pt);
\draw[cA,densely dashed,line width=1.0pt] (0.000,1.868) -- (1.280,1.614) -- (2.560,1.680) -- (3.840,1.705) -- (5.120,1.862) -- (6.400,1.771);
\draw[cA,line width=0.8pt] (-0.060,1.808) -- (0.060,1.928);
\draw[cA,line width=0.8pt] (-0.060,1.928) -- (0.060,1.808);
\draw[cA,line width=0.8pt] (1.220,1.554) -- (1.340,1.674);
\draw[cA,line width=0.8pt] (1.220,1.674) -- (1.340,1.554);
\draw[cA,line width=0.8pt] (2.500,1.620) -- (2.620,1.740);
\draw[cA,line width=0.8pt] (2.500,1.740) -- (2.620,1.620);
\draw[cA,line width=0.8pt] (3.780,1.645) -- (3.900,1.765);
\draw[cA,line width=0.8pt] (3.780,1.765) -- (3.900,1.645);
\draw[cA,line width=0.8pt] (5.060,1.802) -- (5.180,1.922);
\draw[cA,line width=0.8pt] (5.060,1.922) -- (5.180,1.802);
\draw[cA,line width=0.8pt] (6.340,1.711) -- (6.460,1.831);
\draw[cA,line width=0.8pt] (6.340,1.831) -- (6.460,1.711);
\draw[cV,solid,line width=1.2pt] (0.000,1.328) -- (1.280,2.299) -- (2.560,2.207) -- (3.840,2.311) -- (5.120,2.038) -- (6.400,2.168);
\fill[cV] (0.000,1.328) circle (1.9pt);
\fill[cV] (1.280,2.299) circle (1.9pt);
\fill[cV] (2.560,2.207) circle (1.9pt);
\fill[cV] (3.840,2.311) circle (1.9pt);
\fill[cV] (5.120,2.038) circle (1.9pt);
\fill[cV] (6.400,2.168) circle (1.9pt);
\draw[cV,densely dashed,line width=1.0pt] (0.000,1.660) -- (1.280,1.667) -- (2.560,1.791) -- (3.840,1.816) -- (5.120,1.693) -- (6.400,2.070);
\draw[cV,line width=0.8pt] (-0.060,1.600) -- (0.060,1.720);
\draw[cV,line width=0.8pt] (-0.060,1.720) -- (0.060,1.600);
\draw[cV,line width=0.8pt] (1.220,1.607) -- (1.340,1.727);
\draw[cV,line width=0.8pt] (1.220,1.727) -- (1.340,1.607);
\draw[cV,line width=0.8pt] (2.500,1.731) -- (2.620,1.851);
\draw[cV,line width=0.8pt] (2.500,1.851) -- (2.620,1.731);
\draw[cV,line width=0.8pt] (3.780,1.756) -- (3.900,1.876);
\draw[cV,line width=0.8pt] (3.780,1.876) -- (3.900,1.756);
\draw[cV,line width=0.8pt] (5.060,1.633) -- (5.180,1.753);
\draw[cV,line width=0.8pt] (5.060,1.753) -- (5.180,1.633);
\draw[cV,line width=0.8pt] (6.340,2.010) -- (6.460,2.130);
\draw[cV,line width=0.8pt] (6.340,2.130) -- (6.460,2.010);
\draw[cA,line width=1.2pt] (0.550,-1.000) -- (0.950,-1.000);
\fill[cA] (0.750,-1.000) circle (1.9pt);
\node[font=\scriptsize,anchor=west] at (1.030,-1.000) {algebra--calculus};
\draw[cV,line width=1.2pt] (3.550,-1.000) -- (3.950,-1.000);
\fill[cV] (3.750,-1.000) circle (1.9pt);
\node[font=\scriptsize,anchor=west] at (4.030,-1.000) {geometry--combinatorics};
\draw[cGrey!70,solid,line width=1.2pt] (0.550,-1.380) -- (0.950,-1.380);
\node[font=\scriptsize,anchor=west] at (1.030,-1.380) {\textbf{solid}: as published};
\draw[cGrey!70,densely dashed,line width=1.0pt] (3.550,-1.380) -- (3.950,-1.380);
\node[font=\scriptsize,anchor=west] at (4.030,-1.380) {\textbf{dashed}: volume-matched};
\end{tikzpicture}
\caption{\textbf{A sign change that its own control removes.} Co-training value $V$ against the
capacity knob, for the two skill pairs measured on both versions of the axis. \emph{Solid,} as
published, with arms differing by a factor of two in rows: $V$ crosses zero exactly as
Theorem~\ref{thm:limit}(c) predicts when capacity binds. \emph{Dashed,} the same grid re-run
with both arms at $622$ rows: the crossing does not survive for algebra--calculus at all, and
geometry--combinatorics keeps one by a single cell. The distance between a solid curve and its
dashed counterpart is what the volume confound was worth, and it is most of the published
effect. Same data as Table~
ef{tab:capmatched}.}
\label{fig:capacity}
\end{figure}
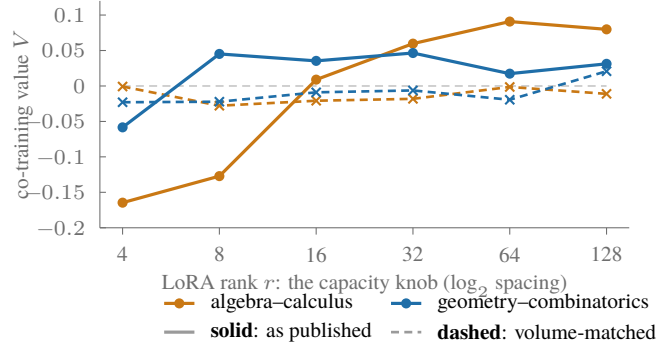

\begin{table}[t]
\centering\small
\caption{\textbf{The capacity axis against its own control} (Qwen2.5-7B, three seeds).
``Published'' varies LoRA rank with arms differing by a factor of two in rows; ``matched'' gives
both arms exactly $622$ rows at every rank. The pair carrying the published effect loses its
crossing entirely; the other keeps one by a single $1.1\sigma$ cell that sits below the
$\mathrm{MDE}_{80}$ gating every other cell in the grid. }
\label{tab:capmatched}
\begin{tabular}{ccccc}
\toprule
LoRA rank $r$ & \multicolumn{2}{c}{algebra--calculus} & \multicolumn{2}{c}{geometry--comb.} \\
\cmidrule(lr){2-3}\cmidrule(lr){4-5}
 & published & \textbf{matched} & published & \textbf{matched} \\
\midrule
4 & $-0.165$ & $\mathbf{-0.0007}$ & $-0.058$ & $\mathbf{-0.0229}$ \\
8 & $-0.127$ & $\mathbf{-0.0278}$ & $+0.045$ & $\mathbf{-0.0222}$ \\
16 & $+0.009$ & $\mathbf{-0.0208}$ & $+0.035$ & $\mathbf{-0.0090}$ \\
32 & $+0.060$ & $\mathbf{-0.0181}$ & $+0.047$ & $\mathbf{-0.0063}$ \\
64 & $+0.091$ & $\mathbf{-0.0014}$ & $+0.017$ & $\mathbf{-0.0194}$ \\
128 & $+0.080$ & $\mathbf{-0.0111}$ & $+0.031$ & $\mathbf{+0.0208}$ \\
\midrule
verdict & \multicolumn{2}{c}{crossing \textbf{vanishes}} & \multicolumn{2}{c}{crossing \textbf{survives}} \\
\bottomrule
\end{tabular}

\end{table}

Budget dilution alone manufactures $V<0$ in the linear model with no capacity constraint at all,
so the burden of proof here is heavy and our data do not meet it. The capacity departure keeps
its theorem and loses its demonstration: on the scorecard, existence of the rank-axis sign
change is \textsc{supported}, its attribution to capacity is \textsc{mixed}.
Figure~\ref{fig:gate} in Appendix~\ref{app:gate} plots the published axis with error bars and
bootstrapped crossing intervals; it is filed there because a full-width float is the wrong
amount of emphasis for an effect this section withdraws. The per-cell values behind both axes,
including each cell's $\mathrm{MDE}_{80}$, are in the released result base rather than printed here, for the same reason.

Relatedly, no scalar property of the base model orders the \emph{sign} of $V$ across families:
parameter count, post-SFT competence, measured committed capacity, and skill-encoding density
each reverse correlation between Qwen2.5 and Qwen3 (Simpson reversals; the two $14$B models from
different families take opposite signs), and matching cells across families on any axis leaves
$|\Delta V|$ at $2.9\times$ the floor. This is a two-family statement, Llama-3 sitting below our
power gate, but it is enough to retire per-model scalar reasoning about co-training sign in
this range \citep{kaplan2020scaling,schaeffer2023mirage}.

\section{What This Does Not License, Limitations, and Conclusion}\label{sec:conclusion}

\subsection{What the idealisation does not license}\label{sec:license}

\paragraph{The effect-size ordering is empirical, not derived.} The geometry suggests
composition above order above arrangement, matching the measured reliabilities
$0.803>0.524>0.197$, and supplies a sufficient condition for its load-bearing part,
$c\rho<\gamma$. Measured on the model itself the condition \emph{fails in all $12$ ordered pairs}
by factors of $3.9$--$13.5$, at three span dimensions and four probe depths. Blaming the skills
does not save it: substituting the \emph{random-subspace} $c$ while keeping the measured $\rho$
and $\gamma$ still fails by $1.5$--$1.8\times$. The
premise is unreachable for any pair of subspaces derived from one base model, so the suppression
mechanism is \emph{untestable} in a single-base-model experiment rather than merely unsatisfied
in ours. The hierarchy this project once led with is an empirical regularity the geometry
motivates and does not license.

\paragraph{The failure ledger.} Six quantitative predictions died and they cluster: existence and
sign claims survived, \emph{location} claims did not. The capacity crossing's location was
excluded with the direction reversed. The transient's crossing budget $T^\star$ fails to rank
crossings on $7$B activations (Spearman $-0.54$, $n{=}6$) \emph{and inside the linear model that
derived it} ($+0.031$, $n{=}278$), so the proxy is structurally wrong rather than mis-transferred.
Four scalar sign-axes for $V$ gave four Simpson reversals. And our own batch-composition reading
of the switch was refuted by the separating cell we built, pure batches in random order landing
with shuffled rather than blocked. The scorecard grades all $26$ pre-registered claims: $18$
supported, $5$ failed, $2$ untested, $1$ mixed.

\paragraph{What the two-regime result does not license.} \emph{The capability term under conflict
reads zero on both corpora, and getting there cost us a published claim.} At eight seeds the
synthetic corpus gives $-0.0052$ on $[-1.77,+1.22]$ floors and numeral-against-spelled $+0.0016$
on $[-0.61,+0.78]$, both containing zero. Neither read that way at three seeds. We twice held a
three-seed conflicted row whose interval excluded zero, and twice the row moved on extension: the
synthetic corpus was written up as spoiled by one outlying seed, and five more seeds refuted that
diagnosis; numeral-against-spelled was then extended under a rule frozen while the job ran
(the registration), its spread nearly tripled to $0.0158$, and the interval
that had excluded zero by $0.01$ floors now contains it. Two of two moved, both towards less
capability effect, and none that excluded zero survived. The claim we kept got stronger in the
same run, the switch row going from $3.86$ to $5.16$ floors with all eight seeds negative.

\emph{The coherent row is not a control for the conflicted ones.} The conflicted corpora and
their control are random halves of one pooled corpus, so their arms share skill composition,
while the coherent grid trains skill $X$ against skill $Y$. Moving between them changes coherence
\emph{and} whether the halves contain different content, and the second change is what forgetting
needs in order to have anything to overwrite. A capability term at the floor on same-distribution
halves is what \emph{both} accounts predict, so its flatness is evidence for neither. The cell
that would make this a genuine two-by-two, two different skills whose halves also disagree, is
not measured here and we know of no construction that makes it: contested outputs require shared
inputs and overwriting requires different ones. Whether those demands are jointly satisfiable is
the sharpest open question this paper leaves. The two regimes are also unequally powered, $18$
coherent cells against three per conflicted row.

\emph{The capability term is not the same object in the two regimes.} In a coherent corpus the
two targets are different skills and their sum is two capabilities added; in a conflicted corpus
they are one capability under two mutually exclusive conventions. We therefore give them two
names, $\mathrm{TOTAL}$ and $\mathrm{SOLVE}$, and no sentence compares them on one scale. What
survives the objection never leaves one regime: within the conflicted corpora alone, a write-time
key moves $\Delta\mathrm{SHARE}$ from $-0.2677$ to $-0.0008$ and $\Delta\mathrm{SOLVE}$ to $0.02$
floors on $[-0.63,+0.60]$, a controlled intervention on the data with the schedule fixed. The
cross-regime contrast is how we found the effect; the key is why we believe it.

\paragraph{Two units, named once.} Significance appears here in two currencies and we do not
silently mix them. A \emph{floor count} divides an effect by the fixed $0.0191$ contrast floor, a
constant estimated once and never refit per row, so it behaves as a Gaussian $z$ against an
assumed scale; every \emph{interval} is a $t$ on that row's own empirical spread. Floor counts are
comparable across tables and intervals are not; intervals carry the row's real uncertainty and
floor counts do not. Where the two disagree we report both, as at $3.86$ floors, which read as a
three-seed $t$ would not clear a $94$-way correction.

\paragraph{Pre-registration, in three tiers.} \emph{Tier 1, frozen before the deciding data
existed:} the volume read-out, the additivity criterion and threshold, the conflict direction
$D<0$, the three anticipated outcomes of the matched capacity axis, and the order predictions with
their abstention threshold, whose sign reversal is why \S\ref{sec:budget} reports a falsification
rather than an adjustment. \emph{Tier 2, redesigned after a diagnosis and re-frozen:} the conflict
instrument (the first build, $238$ rows at one epoch, never made the model adopt the convention,
was graded untested and rebuilt at $864$ rows and three epochs); the reliability estimator (the
first reported exactly $0$ where reliability was low but nonzero, and an independent
recomputation caught it, $0.197$ being the corrected value); and the capacity control. No
conclusion rests on a Tier-2 number whose redesign was aimed at that conclusion.

\emph{Tier 3, and the reason it had to exist.} The tiers above describe our past conduct, and a
defence of a design history is still made by people who have already seen the data. Only a
criterion frozen before data that does not yet exist removes the possibility that iteration shaped
the narrative. Eight are on file, six read out and two registered and unrun (the $\ge6$-skill
additivity test, and the plain-SGD trajectory control,
the registration). Of the six: run-to-run variation over the six skill
pairs read out as \S\ref{sec:budget}'s two
resolutions, and the branch that fired cost us the constant we had been quoting; the seed
extension of the natural row fired R2 and cost us a
capability claim; the synthetic control's extension
(the registration) is complete on one row and stopped early on the
other, reported at five seeds with the deviation and its date in that file; the second-family
switch returned S1, and its Amendment A4, frozen
while the job ran, named the loading confound in advance and is why \S\ref{sec:conflict} prints
the ladder that predicted wrongly beside the row that replicated; the switch on an unconstructed
corpus returned W3, an instrument result rather than a row,
and a defect of our own construction; and the protocol-level replication of the switch
(the registration) returned P1, its P3 branch having been committed in advance to demoting the
switch in this abstract had any replicate mean fallen within a floor of zero. We report the release, the runs
and the defects beside the registrations that fired, because a pre-registration whose inconvenient
instances go unmentioned is not one.

\paragraph{Multiplicity.} This paper reports $94$ quantitative read-outs
(Appendix~\ref{app:census}, which states the counting rule) and makes no family-wise correction
across them. The discovery claims sit at $3.9$--$12.3\sigma$ with unanimous seed signs (nineteen
of nineteen conflict seeds negative, across four scales in two pretraining families, and six of
six more in two independent re-executions of the protocol) and survive
Bonferroni over every read-out in the paper: at $\alpha=0.05$ over $94$ the two-sided threshold is
$z\approx3.5$ and the smallest three-seed firing row is $3.86\sigma$. Null claims are reported as
bounds with their MDE, where the exposure is under-power rather than multiplicity, and everything
near a threshold is treated as unresolved, with no conclusion changing if any of them flips.

\subsection{Limitations and conclusion}

\paragraph{Scope, and where the evidence is thin.} Everything here is measured on competition
mathematics sub-skills, base models, exact-match scoring at $k{=}4$ over $120$ problems per skill,
and a single stateful optimiser. Instruction-tuned checkpoints, preference optimisation, and code,
dialogue and multilingual mixtures are untested, which is to say the post-training settings
industry most cares about are outside what we measured.

\emph{The three legs of the framing are not equally supported, and a reader should weight them
accordingly.} \textbf{Patience} is a transient whose sign crosses zero within a run: a negative
result, well replicated. \textbf{Room} is graded \textsc{mixed} on the scorecard and its effect is
largely explained by a volume confound: motivated, not demonstrated. \textbf{Coherence} is the one
positive finding, and it is the one measured only on conflicts we constructed. The volume-matched
null is one model at one scale for two of its three rows; the domain-boundary control has $n{=}1$;
the cross-family replication is one skill pair on one model. \emph{One mechanism inside our own
coherent result is also undisambiguated}: we did not run blocked training with a replay fraction
of stage-one data in stage two, so ``interleaving'' and ``any schedule that keeps revisiting skill
$A$'' are not separated, and if replay recovers the $+0.0483$ the name changes while the numbers
and the coherence-versus-conflict contrast do not. Finally the theory is a first-order lazy
account whose role is signs and scope boundaries, and full SFT is measurably not lazy
(activation-subspace overlap $0.88$--$0.90$ against LoRA-$r8$'s $0.985$); its quantitative
predictions are the six that died.

\paragraph{Missing cells.} The second-family cell is \emph{closed}: registered before it ran, it
returned S1 (\S\ref{sec:conflict}). Four remain, named where they bind. \textbf{First, a conflict
between two \emph{real} public data sources}, still the one whose failure would change conclusions
rather than narrow them. That cell now carries a specification bought by the attempt that failed:
the two conventions must not be \emph{numerically} equal, since a metric that canonicalises scores
both forms and no arrangement can move an axis the objective cannot see. This excludes
exact-form-against-decimal, $1{,}284$ of the $1{,}284$ pipeline disagreements our survey found.
\textbf{Second, a second domain-boundary pair} behind the ``heterogeneity is not disagreement''
control. \textbf{Third, a $\ge6$-skill additivity test} of the per-skill potential, pre-registered
with an out-of-sample criterion and a withdrawal clause, since the in-sample test the appendix
reports passes too easily. \textbf{Fourth}, registered and unrun rather than forthcoming, the
plain-SGD trajectory control, which would separate a systematic optimiser contribution from
Eq.~\eqref{eq:expansion}. A fifth was open when a referee named it, was run rather than promised,
and is no longer a missing cell: the protocol-level replication of the switch, reported next.

\paragraph{The switch has been held to this paper's own standard.} \S\ref{sec:budget} measures
run-to-run variation for \emph{order} by executing the whole protocol independently, and that
measurement cost us a constant. The switch had not had the same treatment: its headline row
rested on eight seeds inside one execution, while this paper elsewhere shows run-level noise
reaching four to eight times seed-level. Two replicate corpora were therefore rebuilt from the
same pool under different split seeds and $48$ trainings run against a read-out frozen before any
of them started, with the demotion branch committed in advance: any replicate mean within one
floor of zero, or reversing sign, would have put the switch in the class of the $0.0619$ order
winner we use to argue that recipe search cannot be trusted, and demoted it in the abstract. It
did not happen. All three replicate means are negative beyond four floors, all six new seeds are
negative, and the re-execution spread is $0.087$ against the selected order cell's $0.144$. The
one result this paper leans on is the one result it tried hardest to break.

\paragraph{What to take away.} Recipe search is bounded by two measured numbers: the span of the
reachable set inside a coherent domain at fixed volume, whose largest single contrast is $0.0619$,
and the resolution at which anyone can tell two of its points apart, $0.054$ for a three-seed
design and $\approx0.14$ for the single runs a practitioner actually spends. The failure of recipe
search here is not that effects are absent but that they are \emph{unresolvable}, with a specific
signature: the grid's best cell clears the optimistic threshold and then reverses sign when the
identical protocol is run again. A ranking that is its own noise cannot be repaired by more
searching, only by enlarging the set or buying resolution, and we price both.

That is the story for \emph{order} and not for \emph{arrangement}, where the difference is a
coordinate. Scored on the skill a contrast targets, arrangement sits at $0.0140$ against a
$0.0191$ floor; scored on the two skills the arms were trained for it moves $+0.0483$ at
$t=8.06$, about $+0.024$ per skill, still roughly half of what doubling one skill's data buys, so
volume remains the better purchase. What arrangement has that volume does not is a \emph{switch}:
internal disagreement, not difficulty and not domain distance, changes which quantity arrangement
moves at all, from capability to allocation, and moves the second by two orders of magnitude;
writing the convention into the input at training time turns it off.

\paragraph{The screening question, and the reason to expect less of it than we first did.} The
question the switch isolates is cheap to ask: \emph{would the two corpora ever label the same
input differently?} \textbf{If no}, do not spend runs on arrangement; at these scales they will
measure their own noise, and the budget belongs to volume. \textbf{If yes and the conflict is
resolvable at write time}, resolve it there, by keying the convention into the input or cleaning
the disagreement out; one intervention on the data returned both terms to zero here and no
schedule matched it. \textbf{If yes and it cannot be resolved}, then and only then is arrangement
the operative lever, acting on allocation rather than capability.

\emph{Our own base rates argue against expecting much of this in practice, and we would rather say
so than let a reader assemble the paradox.} Convention conflict is $9$ in $92{,}359$ between two
independently labelled public corpora and $1{,}284$ in $65{,}456$ between two derivation
pipelines, and \emph{every one of those $1{,}284$} is exact form against decimal, which our own
scorer treats as equal and which therefore cannot fire the switch at all. The commonest real
convention conflict in public mathematics data is invisible to the metric this literature uses.
Detecting it needs a scorer that distinguishes $0.75$ from $\tfrac{3}{4}$, at which point the
disagreement becomes visible to the objective and the switch becomes available; whether that trade
is worth making is a question about what a user of the model wants, not one our measurements
answer. The screening question is therefore validated on constructed conflicts in mathematics SFT,
and we advance it as a procedure worth testing elsewhere rather than as advice that transfers.

\looseness=-1
Three things can still be said. \emph{The literal test is cheap where it applies}: group by
normalised input and look for distinct normalised outputs. That finds contested \emph{examples},
what the static-analysis literature catalogues
\citep{dsouza2025disagreement,incompletelearning2026}, but misses what fires the switch here,
since a convention is a property of a \emph{source} and two sources can disagree about every
answer's form while sharing no input. \emph{The axis to inspect is output format conditional on
matched inputs}: units, numeral form, answer wrapper, code style, citation shape, since that is
what $\kappa_A\neq\kappa_B$ is, while heterogeneity of \emph{content} is inert here. \emph{And one
test needs no prior knowledge of the conflict.} Score the blocked and interleaved arms twice, once
committed to one convention and once at the better candidate per problem
(Appendix~\ref{app:agnostic}): if the two rules disagree about which arm won, the corpus contains
a convention conflict; if they agree, the difference is capability. We report this as the
operational form of the result rather than a validated tool.

\enlargethispage{4\baselineskip}
The idealisation is honest about its reach: it predicted signs and existences that survived
measurement, every location it ventured was wrong, and the condition that would have given it the
effect-size ordering fails on our own instrument. A thought experiment is not there to be right.
It is there to make the departures measurable.

\clearpage
\bibliographystyle{unsrtnat}
\ifanon\bibliography{references_anon}\else\bibliography{references}\fi

\clearpage
\appendix

\section{The Two Ledgers: Registered Claims and Machine Checks}\label{app:scorecard}

This project keeps two itemised ledgers and neither is a summary of the other. The first grades
every claim we registered against the data that decided it. The second checks every theoretical
statement numerically before it is used, so that a claim holding only at convergence or only in
exact arithmetic fails visibly: the checks run gradient descent step by step against the closed
forms with hard-coded tolerances, and four of the twenty-five rejected drafts of our own
statements, so every statement in the main text is the post-rejection version.

\subsection*{The public scorecard}

Every pre-registered claim, its verdict, and the evidence that decided it, generated from the result files by a ledger script and reproduced verbatim so
that no row can be silently dropped between versions. Tier-1/Tier-2
status per \S\ref{sec:license}: rows whose instrument was rebuilt after a diagnosis say so in
their evidence column.

\newcommand{\scorecardcaption}{\textbf{Public scorecard of every registered claim} ($18$
supported, $5$ failed, $2$ untested, $1$ mixed). Four \textsc{failed} rows are quantitative
localisation claims; the fifth is a mechanism reading this paper published and then refuted
(\S\ref{sec:license}); one \textsc{untested} row needs a design not built and the other is a
registered experiment whose control fired, so its registration permits no read-out; the \textsc{mixed}
row is a claim whose controlled re-run split between its two skill pairs
(\S\ref{sec:capacity}).}
{\centering\scriptsize
\begin{longtable}{l>{\raggedright\arraybackslash}p{0.27\linewidth}>{\raggedright\arraybackslash}p{0.46\linewidth}}
\caption{\scorecardcaption}\label{tab:scorecard}\\
\toprule
status & claim & evidence \\
\midrule
\endfirsthead
\multicolumn{3}{l}{\emph{\tablename~\thetable\ (continued)}}\\
\toprule
status & claim & evidence \\
\midrule
\endhead
\midrule\multicolumn{3}{r}{\emph{continued on the next page}}\\
\endfoot
\endlastfoot
\textsc{failed} & within-batch composition carries the conflict switch (batch purity, not stage order, is what blocking buys) & job 1725, 3 seeds, pre-registered share rule: the missing cell (PURE batches, RANDOM order) lands at acc\_B 0.2242 vs shuf 0.2179 and blocked 0.4454, so purity share = 0.03 -- between-batch order carries ~97\% of the span. Control span $+0.0037$, unresolvable at kappa~0 as predicted. The whole conflict effect sits at the STAGE timescale the theory models \\
\addlinespace[0.3pt]
\textsc{failed} & r* identical for both pairs (factorised account) & 3 seeds, r* 15.2 [12.2,24.4] vs 5.9 [4.7,8.4]; ratio CI [0.231,0.583] excludes 1.0 \\
\addlinespace[0.3pt]
\textsc{failed} & r* ratio = m\_hat ratio = 1.271 (CCH-style load accounting) & 3 seeds, ratio 0.382, 95\% CI [0.231, 0.583] excludes 1.271; direction also reversed (higher-demand pair crosses EARLIER) \\
\addlinespace[0.3pt]
\textsc{failed} & V collapses on measured committed capacity across families & matched cross-family pairs differ by 2.9x the noise floor; families separate more at higher committed capacity \\
\addlinespace[0.3pt]
\textsc{failed} & V collapses on skill-encoding DENSITY across families & the favourable read-out was a floor artefact: matched-pair median |dV| 0.0429 (1.83x) rises to 0.0680 (2.9x) once cells below acc\_only 0.10 are gated out -- indistinguishable from committed capacity (0.0688); within-family sign still reverses on all three density statistics \\
\addlinespace[0.3pt]
\textsc{mixed} & the rank-axis sign change is ATTRIBUTABLE TO A CAPACITY CONSTRAINT & job 1691, volume-matched axis (both arms 622 rows at every rank, 3 seeds). algebra-calculus -- the pair carrying the published swing from $-0.165$ to $+0.080$ -- keeps ONE sign across the whole grid once volume is matched (pooled V $-0.013$, crossing VANISHES). geometry-combinatorics fires the registered survives-rule on a seed-consistent $+0.021$ at r=128 ($+0.006$/$+0.035$/$+0.021$), but that cell is 1.1 sigma, under the MDE80 = 0.038 every capacity cell is gated by, and it sits >10x from the published r* = 5.9. SETTLED: the published magnitudes were mostly the volume confound. OPEN: whether a small real gate survives at matched volume, below this design's power \\
\addlinespace[0.3pt]
\textsc{untested} & joint training taxes identification of the partner's PRIVATE structure & fell out of the machine-caught correction to Theorem 2(b) \\
\addlinespace[0.3pt]
\textsc{supported} & a directional order effect exists, and prerequisite-first wins once data volume is matched & balanced level-2: 6/6 pairs measurable, all prerequisite-first (alg-calc $+0.0743$ = 3.9 sigma); 5-seed ICC 0.524 (F=4.62); unbalanced sign was a volume artefact \\
\addlinespace[0.3pt]
\textsc{supported} & the arrangement contrast is ~4x less reliable than the composition contrast & 5 seeds: ICC ORD/D 0.197 (F=1.82 < F.05=2.00, n.s.) vs V 0.803 and FWT 0.820 (56 cells, kbar 3.34, 187 obs); the ordering also holds on MAGNITUDE (mean |V| 0.0416 > |ASYM| 0.0269 > |D| 0.0170, SEs 0.003/0.003/0.001); sign split 84/71, p=0.335 -- no consistent direction \\
\addlinespace[0.3pt]
\textsc{supported} & V/FWT are the reliable observables & 5 seeds: ICC 0.803 / 0.820, F 14.6 / 16.2 vs F.05 = 2.00 \\
\addlinespace[0.3pt]
\textsc{supported} & parameter count is NOT the control parameter & opposite V signs at identical 14B across families ($-0.117$ qwen3 vs $+0.092$ qwen2.5, alg$\to$calc); pooled spearman $+0.084$ on 187 triples is slice-dependent ($+0.559$ qwen2.5 vs $+0.008$ qwen3; $+0.39$ restricted to one direction) and is NOT what the claim rests on \\
\addlinespace[0.3pt]
\textsc{supported} & sign and magnitude load on DIFFERENT factors & sign scale 82.2\%/pair 12.1\% vs |V| 37.6\%/21.1\% \\
\addlinespace[0.3pt]
\textsc{supported} & |V| runs opposite to overlap (overlap = little new information) & 4/4 cells \\
\addlinespace[0.3pt]
\textsc{supported} & effective rank cannot serve as C (Johnson-Lindenstrauss) & C\_eff(f)/C(0)=0.979..1.027 with masks asserted; synthetic reproduction \\
\addlinespace[0.3pt]
\textsc{supported} & linear capacity map C(1-f) quantitatively falsified & no y* places peaks both close and in-window \\
\addlinespace[0.3pt]
\textsc{supported} & the sign gate replicates across seeds & THREE seeds: 11 of 12 (pair, rank) cells agree in sign; the one that does not is alg-calc r=16, the grid point adjacent to that pair's crossing (per-seed $-0.010$/$+0.006$/$+0.031$, mean $+0.009$ inside the noise floor) -- where the crossing is \\
\addlinespace[0.3pt]
\textsc{supported} & Theorem 1 arrangement invariance: exact identity, machine-verified & 13/13 checks; residual 1e-15 incl. on real GD solutions \\
\addlinespace[0.3pt]
\textsc{supported} & Theorem 2 capacity sign gate EXISTS at fixed model size & 824: V flips sign in BOTH pairs, 11/12 cells across three seeds, on the published volume-UNmatched arms; job 1691 then reattributes the magnitudes -- see the mechanism row \\
\addlinespace[0.3pt]
\textsc{supported} & V(r) monotone, no interior peak & no significant decrease survives the paired SEs \\
\addlinespace[0.3pt]
\textsc{supported} & Theorem 3 retention formula (contested-subspace localisation) & closed-form residual 1e-15; free data test in section B \\
\addlinespace[0.3pt]
\textsc{supported} & |D| grows when laziness is broken (higher lr) & |D| 0.0236 $\to$ 0.0507 $\to$ 0.0653 monotone across lr 1e-5/3e-5/1e-4; SE ~0.0191; all three points at n=3 \\
\addlinespace[0.3pt]
\textsc{supported} & RECENCY transfers per-skill (hub geometry) & volume-matched additivity 0.073 vs threshold 1.0 (contaminated gave 1.109); per-skill potential algebra > geometry > calculus ~ comb fits all 6 ASYMs; seed-permutation self-check passed \\
\addlinespace[0.3pt]
\textsc{supported} & label conflict revives the arrangement effect (D<0) & redo with a pooled 864-row conflict, both budgets: conflict D = $-0.1425$ (7.5 sigma) and $-0.2275$ (11.9 sigma) against control $+0.0017$ and $-0.0037$; difference $-0.2238$ (8.3 sigma). First attempt floored out and is kept as an apparatus failure \\
\addlinespace[0.3pt]
\addlinespace[0.3pt]
\textsc{supported} & the conflict switch survives an independent execution of its own protocol & two replicate corpora rebuilt from the same pool under different split seeds (51.7\% and 49.2\% overlap with the published halves), 48 trainings. Guard and separability pass on both. Replicate means $-0.0985$ (published, 8 seeds), $-0.1258$, $-0.0825$; all six new seeds negative; controls inert at 0.26 and 0.04 floors. $\delta_{\text{run}}^{\text{switch}}=0.087$ against the 0.144 the search-selected order cell carries, so P1 fires and the P3 demotion branch, committed in advance, does not \\
\addlinespace[0.3pt]
\textsc{supported} & the switch is not one pretraining family & Qwen3-8B base, five seeds, separability checked first and passing at infinity: guard 0.0520, $D=-0.0875$ at 4.58 floors, all five seeds negative against an inert control. Amendment A4, frozen while the job ran, named the guard-ladder counterexample in advance and it duly appeared \\
\addlinespace[0.3pt]
\textsc{untested} & the switch fires on a conflict drawn from real public data & W3 fired, the control moving at 4.45 floors, so the registration permits NO read-out on the wild row and the cell is not closed. Diagnosed afterwards: the two conventions the wild corpus pairs are ones the scorer treats as equal, so the contrast was unmeasurable by construction. Reported as a result about exact-match evaluation, not about arrangement \\
\midrule
\multicolumn{3}{l}{FAILED 5, MIXED 1, SUPPORTED 18, UNTESTED 2} \\
\bottomrule
\end{longtable}

\par}

\subsection*{The machine checks}\label{app:machine}

Table~\ref{tab:appmachine} lists them per statement.

\begin{table}[ht]
\centering\footnotesize
\caption{\textbf{Per-statement numerical checks.} }
\label{tab:appmachine}
\begin{tabular}{clcl}
\toprule
\# & statement checked & criterion & outcome \\
\midrule
1 & convergence map: GD limit $=$ closed form & $\|\theta_{\mathrm{GD}}-\theta_{\text{closed}}\|<10^{-6}$ & pass \\
2 & arrangement identity on real GD solutions & residual $<10^{-8}$ & pass \\
3 & $D=0$ for orthogonal skills & $<10^{-10}$ & pass \\
4 & $|D|$ grows with the principal angle & monotone in $c$ & pass \\
5 & capacity gate (b) as first drafted: selection alone gives $V>0$ & Monte-Carlo sign & \textbf{rejected} \\
6 & per-skill transferability as first argued & split vs collapse & \textbf{rejected} \\
7 & retention closed form & residual $<10^{-8}$ & pass \\
8 & conflict drives $D$ negative & monotone in $\kappa$ & pass \\
9 & $\Delta(T)\to P_W(I-P_A)$ & residual $2\times10^{-11}$ & pass \\
10 & blocked never worse when $\Sigma$'s commute & zero AM--GM violations & pass \\
11 & interior maximum of $|D(T)|$ & peak strictly inside & pass \\
12 & two-term expansion, Eq.~\eqref{eq:expansion} & rel.\ residual $<0.3\%$ & pass \\
13 & commutator/dissimilarity ratio linear in $T$ & $0.63\to20.3$ over $T{=}1$--$32$ & pass \\
14 & $\|g_D\|\le c\rho\|\delta\|$, constant sharp & $0/400$; ratio $1.000$ & pass \\
15 & $c\rho<\gamma\Rightarrow\|g_D\|<\|g_V\|$ & $1355/1355$; bites & pass \\
16 & commutator norm $=\max_i\cos\theta_i\sin\theta_i$ & $0/200$; $0$ if nested & pass \\
17 & $|V|,|D|\propto\varepsilon_B$, $|\ASYM|$ affine & $300/300$, res.\ $<10^{-9}$ & pass \\
18 & the $\varepsilon^\star$ crossing lands where predicted & $0/193$ mis-sided & pass \\
19 & driver norms alone imply the three-way ordering & $3.4\%$ of draws & \textbf{rejected} \\
20 & $\ASYM(T)$ changes sign, $\eta<1/\lambda_{\max}$ & $216/300$ draws & pass \\
21 & $D(T)$ changes sign & $211/300$ draws & pass \\
22 & $T^\star$ ranks the crossing budget & Spearman $+0.03$, $n{=}278$ & \textbf{rejected} \\
23 & the crossing is a rotation of both bilinear factors & $11.2^\circ$ vs $11.7^\circ$ & pass \\
24 & a base-state score inverts rather than fades & $28.7\%$, below chance & pass \\
25 & the joint arm restores the inverted read-out & $+62.7$ points; oracle gap $26$ & partial \\
\bottomrule
\end{tabular}
\end{table}

\section{The Room Axis: Its Premise Measured and Its Result Plotted}\label{app:premise}

\subsection*{The sufficient condition, measured on the model}

\S\ref{sec:license} states that $c\rho<\gamma$ fails in $36$ of $36$ ordered pairs; this
appendix gives the measurement. $\rho$ and $\gamma$ are ratios of projections of the
fine-tuning displacement (Frobenius norms of $\Delta W$ at the probed layer's input space),
$c$ the largest principal-angle cosine between skill activation spans; $k$ is the span
dimension and the one free parameter, so the measurement is reported across a range. The projection algebra is verified against the definitions to $10^{-7}$. Table~\ref{tab:apppremise} reports it.

\begin{table}[ht]
\centering\small
\caption{\textbf{The premise fails at every span dimension and every depth we can measure}
(Qwen2.5-7B, learning budget, averaged over the twelve ordered skill pairs). Left: span
dimension $k$ at layer fraction $0.75$. Right: depth at $k{=}1024$. The ratio $c\rho/\gamma$
must fall below $1$ for the theorem to bite; it never falls below $3.9$, and $c\ge0.999$
everywhere, the four skills' representation spans being all but identical, which is the
coherent-domain scope condition written in subspace terms.}
\label{tab:apppremise}
\begin{tabular}{rrrrr}
\toprule
$k$ & $c$ & $\rho$ & $\gamma$ & $c\rho/\gamma$ \\
\midrule
$64$ & $0.9994$ & $0.9958$ & $0.0738$ & $13.5$ \\
$256$ & $0.9998$ & $0.9876$ & $0.1337$ & $7.4$ \\
$1024$ & $0.9999$ & $0.9616$ & $0.2467$ & $3.9$ \\
\bottomrule
\end{tabular}
\hspace{1.5em}
\begin{tabular}{rrrrr}
\toprule
layer & $c$ & $\rho$ & $\gamma$ & fails \\
\midrule
$7$ & $1.0000$ & $0.9568$ & $0.2504$ & $36/36$ \\
$14$ & $0.9991$ & $0.9607$ & $0.2462$ & $36/36$ \\
$20$ & $0.9999$ & $0.9616$ & $0.2467$ & $36/36$ \\
$24$ & $0.9999$ & $0.9584$ & $0.2493$ & $36/36$ \\
\bottomrule
\end{tabular}
\end{table}

The margin shrinks with $k$ roughly as $k^{-0.45}$, and the tempting extrapolation to ratio
$\approx1$ at full rank is meaningless: at $k=d$ both $\rho$ and $\gamma$ are identically zero
and the ratio is $0/0$. The premise question is only posed where a span carries information,
and at every such cut it fails by $3.9\times$ to $13.5\times$.

\subsection*{The axis and its control, plotted}\label{app:gate}

Figure~\ref{fig:capacity} plots it.

The sign change Theorem~\ref{thm:limit}(c) requires, and the volume-matched control that removes
it for the pair that carried it. Kept because the sequence of effect, decomposition and control is
the reusable part, and out of \S\ref{sec:capacity} because a full-width figure is the wrong amount
of emphasis for a result that section withdraws.

\begin{figure}[t]
\centering
\begin{tikzpicture}[font=\small,>=Stealth,scale=1.0]
\def\xo{0.6}\def\yo{2.2}\def\sx{1.55}\def\sy{11.0}
\draw[cGrey,thin] (\xo-0.25,0) -- (\xo+8.6,0);
\draw[cGrey,thin] (\xo-0.25,0) -- (\xo-0.25,3.7);
\foreach \i/\r in {0/4,1/8,2/16,3/32,4/64,5/128}
 {\draw[cGrey,thin] (\xo+\i*\sx,0) -- (\xo+\i*\sx,-0.09);
 \node[cGrey,font=\scriptsize,anchor=north] at (\xo+\i*\sx,-0.12) {\r};}
\node[cGrey,font=\scriptsize,anchor=north] at (\xo+3.9,-0.52)
 {LoRA rank $r$: the capacity knob ($\log_2$ spacing)};
\foreach \v in {-0.15,-0.10,-0.05,0,0.05,0.10}
 {\draw[cGrey,thin] (\xo-0.25,\yo+\v*\sy) -- (\xo-0.16,\yo+\v*\sy);
 \node[cGrey,font=\scriptsize,anchor=east] at (\xo-0.3,\yo+\v*\sy) {$\v$};}
\node[cGrey,font=\scriptsize,rotate=90,anchor=south] at (\xo-1.15,1.85)
 {$V=\mathrm{acc}_B(\shuf)-\mathrm{acc}_B(\only)$};
\draw[cGrey,densely dashed] (\xo-0.25,\yo) -- (\xo+8.6,\yo);
\node[cGrey,font=\scriptsize,anchor=north east] at (\xo+8.55,\yo-0.06) {$V=0$};
\draw[cV,line width=1.2pt,solid] (0.600,0.390) -- (2.150,0.802) -- (3.700,2.299) -- (5.250,2.857) -- (6.800,3.201) -- (8.350,3.078);
\fill[cV] (0.600,0.390) circle (2.1pt);
\draw[cV,thin] (0.600,0.179) -- (0.600,0.600);
\fill[cV] (2.150,0.802) circle (2.1pt);
\draw[cV,thin] (2.150,0.592) -- (2.150,1.012);
\fill[cV] (3.700,2.299) circle (2.1pt);
\draw[cV,thin] (3.700,2.089) -- (3.700,2.509);
\fill[cV] (5.250,2.857) circle (2.1pt);
\draw[cV,thin] (5.250,2.647) -- (5.250,3.067);
\fill[cV] (6.800,3.201) circle (2.1pt);
\draw[cV,thin] (6.800,2.991) -- (6.800,3.411);
\fill[cV] (8.350,3.078) circle (2.1pt);
\draw[cV,thin] (8.350,2.868) -- (8.350,3.289);
\draw[cA,line width=1.2pt,densely dashed] (0.600,1.558) -- (2.150,2.697) -- (3.700,2.590) -- (5.250,2.712) -- (6.800,2.391) -- (8.350,2.544);
\fill[cA] (0.600,1.558) circle (2.1pt);
\draw[cA,thin] (0.600,1.348) -- (0.600,1.768);
\fill[cA] (2.150,2.697) circle (2.1pt);
\draw[cA,thin] (2.150,2.486) -- (2.150,2.907);
\fill[cA] (3.700,2.590) circle (2.1pt);
\draw[cA,thin] (3.700,2.379) -- (3.700,2.800);
\fill[cA] (5.250,2.712) circle (2.1pt);
\draw[cA,thin] (5.250,2.502) -- (5.250,2.922);
\fill[cA] (6.800,2.391) circle (2.1pt);
\draw[cA,thin] (6.800,2.181) -- (6.800,2.601);
\fill[cA] (8.350,2.544) circle (2.1pt);
\draw[cA,thin] (8.350,2.334) -- (8.350,2.754);
\node[cV!85!black,font=\scriptsize,anchor=west] at (8.490,3.078) {algebra$\to$calculus};
\node[cA!85!black,font=\scriptsize,anchor=west] at (8.490,2.444) {geometry$\to$comb.};
\fill[cV!85!black] (3.587,2.200) circle (2.6pt);
\draw[cV!85!black,line width=1.1pt] (3.098,2.200) -- (4.647,2.200);
\draw[cV!85!black] (3.098,2.100) -- (3.098,2.300);
\draw[cV!85!black] (4.647,2.100) -- (4.647,2.300);
\node[cV!85!black,font=\scriptsize,anchor=north] at (3.887,2.000) {$r^\star=15.2$};
\fill[cA!85!black] (1.476,2.200) circle (2.6pt);
\draw[cA!85!black,line width=1.1pt] (0.976,2.200) -- (2.247,2.200);
\draw[cA!85!black] (0.976,2.100) -- (0.976,2.300);
\draw[cA!85!black] (2.247,2.100) -- (2.247,2.300);
\node[cA!85!black,font=\scriptsize,anchor=south] at (1.176,2.400) {$r^\star=5.9$};
\node[cGrey,font=\scriptsize,anchor=north] at (4.35,1.62) {bars: $95\%$ bootstrap CI};
\node[anchor=north west,font=\scriptsize,align=left,text width=9.4cm] (vex) at (\xo-0.35,-0.85)
 {\textbf{\textcolor{cOK}{Existence: confirmed.}} Both pairs cross zero, monotonically and with
 no interior peak; $11$ of $12$ (pair, rank) cells agree in sign across three seeds, the
 exception being the rank adjacent to the crossing.};
\node[anchor=north west,font=\scriptsize,align=left,text width=9.4cm]
 at ([yshift=-3pt]vex.south west)
 {\textbf{\textcolor{cNO}{Location: falsified, and inverted.}} Part (c) predicts the
 \emph{higher-demand} pair (geometry--comb., $\hat y{=}2.65$ vs.\ $2.08$) crosses \emph{later}:
 the predicted ratio $r^\star_{\text{geo,comb}}/r^\star_{\text{alg,calc}}$ is $1.271$ under
 load accounting, $1.0$ if pair-independent. Measured: $0.382$, $95\%$ CI $[0.231,0.583]$,
 both excluded, and the direction reversed.};
\end{tikzpicture}
\caption{\textbf{The capacity departure: the predicted sign change appears, and its own control
withdraws most of it.} $V(r)$ at fixed model (Qwen2.5-7B) and fixed data, varying only the LoRA
rank. Points are three-seed means with problem-level standard errors; filled markers on the axis
are bootstrapped zero-crossings with $95\%$ intervals ($4000$ draws resampling both problems and
seeds). The sign change Theorem~\ref{thm:limit}(c) requires is reproduced for both pairs. Read
with Table~\ref{tab:capmatched}, which matches the arms' data volume and removes the crossing
entirely for the pair that carried it.}
\label{fig:gate}
\end{figure}

\section{The Conflict Rows Under Convention-Agnostic Scoring}\label{app:agnostic}

\paragraph{The objection.} In every conflicted corpus the blocked arm trains convention $A$
first and convention $B$ last, and the committed read-out $D=\mathrm{acc}_B(\text{shuf}) -
\mathrm{acc}_B(A{\to}B)$ scores exactly the convention the blocked arm wrote in its final stage.
Blocking's win could therefore be the evaluation siding with the last stage, a constructed
alignment rather than a property of the arms.

\paragraph{The instrument.} Both conventions are evaluated on the \emph{same} problems in the
\emph{same} order (the analyzer re-asserts this on the eval files before reading any number), so
a scoring rule with no favourite convention exists without new inference: count each problem at
the rate of its better-scoring convention, $\mathrm{agn} = \tfrac1n\sum_i \max(p^A_i, p^B_i)$
with $p^A_i,p^B_i$ the $k{=}4$ pass fractions under each convention, and read
$D_{\mathrm{agn}} = \mathrm{agn}(\text{shuf}) - \mathrm{agn}(A{\to}B)$ on the very runs that
produced the committed $D$. Three read-outs then separate three stories. $D$ committed answers
\emph{which convention} wins; $\Delta$SOLVE answers \emph{how much is solved}, and is conserved
when samples merely reallocate between conventions on one problem; $D_{\mathrm{agn}}$ moves if
and only if commitment itself has a price. A row with SOLVE at the floor and $D_{\mathrm{agn}}$
firing is a shuffled arm splitting its four samples across conventions \emph{on the same
problem}, not one solving fewer problems.

\begin{table}[H]
\centering\small
\caption{\textbf{The committed win is allocation where both conventions are correct, and
per-problem mixing where they are not.} $D$ committed is the main text's contrast,
$\Delta\mathrm{SOLVE}$ its capability sum, $D$ agnostic the best-convention-per-problem
contrast. Parentheses give distance from zero in $0.0191$ floors; bold marks $\ge2$ floors. The
numeral rows at $7$B and the synthetic row carry eight seeds; the two scale rows were never
extended and carry three. The three-seed synthetic subset previously printed, $-0.2275$ at
$11.91$ floors, had a $t$ of $-88.8$ against the eight-seed $-33.4$, which is why the count is
now stated beside every row.}
\label{tab:agnostic}
\begin{tabular}{lcccc}
\toprule
conflicted corpus & $n$ & $D$ committed & $\Delta\mathrm{SOLVE}$ & $D$ agnostic \\
\midrule
numeral against spelled, 7B & $8$ & $\mathbf{-0.0985}$ ($\mathbf{5.16}$) & $+0.0016$ ($0.09$) & $-0.0084$ ($0.44$) \\
Arabic against Roman, 7B & $3$ & $\mathbf{-0.1046}$ ($\mathbf{5.48}$) & $\mathbf{+0.0500}$ ($\mathbf{2.62}$) & $+0.0308$ ($1.61$) \\
synthetic +1, 7B & $8$ & $\mathbf{-0.2346}$ ($\mathbf{12.29}$) & $-0.0052$ ($0.27$) & $\mathbf{-0.0887}$ ($\mathbf{4.65}$) \\
numeral against spelled, 3B & $3$ & $-0.0263$ ($1.38$) & $-0.0071$ ($0.37$) & $+0.0004$ ($0.02$) \\
numeral against spelled, 14B & $3$ & $\mathbf{-0.2037}$ ($\mathbf{10.67}$) & $-0.0312$ ($1.64$) & $\mathbf{-0.0496}$ ($\mathbf{2.60}$) \\
\bottomrule
\end{tabular}

\end{table}

\paragraph{Reading, row by row.} At $7$B, where both conventions are correct, the committed wins
of $-0.0985$ and $-0.1180$ become $-0.0084$ ($0.44$ floors) and $+0.0308$ ($1.61$ floors) under
agnostic scoring: an order of magnitude smaller and inside two floors, so blocking solves
\emph{nothing} the shuffled arm cannot solve and only commits, which is \S\ref{sec:conflict}'s
claim re-derived on a metric that cannot be accused of siding with the last stage. (The Roman
row's $\Delta$SOLVE, $+0.0465$ at $2.43$ floors with five seeds, is previously unreported and is
given as measured: if anything the \emph{shuffled} arm solves more there, opposite to a
capability story for blocking. Its controls read null with $|t|<0.7$ and per-seed spreads up to
$\pm0.11$, so we quote it and decline to lean on it.) The two rows where $D_{\mathrm{agn}}$ does
fire, synthetic $+1$ at $-0.0887$ and $14$B at $-0.0496$, are precisely those where SOLVE stays
low ($0.27$ and $1.64$ floors): the shuffled arm loses its best convention per problem while its
sum holds, which is mixing within a problem, the compressive-write signature the path-level
the path-level account predicts \citep{dpd2026}. Where the conflict is between two \emph{correct} formats,
commitment is free and the whole committed effect is allocation; where it contests answer
\emph{values}, refusing to commit costs the shuffled arm its per-problem best while costing it
nothing in sum. No row supports ``blocking teaches the model more mathematics''.

\section{The Per-Skill Order Potential}\label{app:potential}

Held out of \S\ref{sec:budget} because its evidential weight does not match its visual weight:
it is a three-parameter fit to six observations at one budget, and four of those six pairs
reverse sign at the other budget we measured. It is reported in full because it would be the
cheapest thing in this paper if it replicated at $\ge6$ skills.

At one stated budget the six matched pairwise effects are consistent with a per-skill potential,
$\ASYM(X,Y)\approx\phi_X-\phi_Y$ ($\phi$: algebra $0>$ geometry $-0.032>$ calculus
$-0.072\approx$ combinatorics $-0.081$), and the pre-registered additivity-over-triples
criterion passes at $0.073$ against a threshold of $1.0$. On unmatched data the same test reads
$1.109$ and the recovered ordering inverts, which is the sharpest argument we have that volume
matching is mandatory.

\emph{The structure predicts out of sample, which is the part worth having.} Additivity over
triples is a goodness-of-fit statistic on the data $\phi$ was fitted to, and on its own weak
evidence: three free parameters against six observations. So we froze $\phi$, fitted on
\emph{pairs} only, wrote down the implied ranking of all six three-skill orderings, and trained
them. The prediction file is timestamped $1$h$46$m before the first three-stage run existed, and
the scoring script fits nothing. Both read-outs were fixed in that file: P1, Spearman
$\rho \ge 0.60$ between the predicted $\phi_{\text{first}}-\phi_{\text{last}}$ and the measured
$\mathrm{TOTAL}$; P2, the predicted-best and predicted-worst orderings separated by more than
$2$ s.e. Table~\ref{tab:orderperm} gives $\rho=0.657$ and a best-to-worst gap of
$0.1215\pm0.0331$, $3.7$ s.e., so both pass: a quantity fitted on pairs ranked sequences it had
never seen. The reduction is therefore measured rather than conditional, and \textbf{order
search costs $n$ per-skill measurements rather than $n!$ trials}. That is the one constructive
result this paper sets against \S\ref{sec:budget}'s wall, which says a \emph{search} over
recipes cannot pay for itself; the potential says the search is unnecessary where it would have
been run.

Three reasons still keep the structure held lightly, none repaired by
Table~\ref{tab:orderperm}. The ranking is not monotone, the predicted-third ordering outscoring
the predicted-second by $0.036$, which is why $\rho$ is $0.657$ and not $1$, and only the
extremes are separated by more than noise. Four of the six pairs reverse sign between our two
budgets, so $\phi$ is a snapshot at its stated budget, to be refit at the budget where it will
be used. And the additivity residual beats the marginal noise floor only because same-seed arms
share checkpoints, a permutation check anyone reporting a suspiciously good fit should run. A
$\ge6$-skill replication is now registered with the criterion that matters
: the fit sees ten pairs, five are held out entirely, the
verdict turns on out-of-sample sign and error rather than the in-sample residual, and a clause
commits to \emph{withdrawing} this appendix rather than softening it if the held-out test fails.
What is established here is narrower than the additive form: the fitted potential has
out-of-sample ordering power at $n=3$, at one budget, for four skills.

\begin{table}[H]
\centering\small
\caption{\textbf{A potential fitted on pairs ranks orderings it never saw.} All six orderings of
three skills, trained to the same budget, scored by $\mathrm{TOTAL}$ at the final checkpoint, so
an ordering is rewarded for what it \emph{retains}, not for what it learned last. The
predicted column is $\phi_{\text{first}}-\phi_{\text{last}}$ from a file frozen before any three-stage sequence was trained;
the middle position cancels at $n=3$). Qwen2.5-7B, three seeds, per-cell s.e.\ $0.0234$.
The analysis measures and compares and fits
nothing.}
\label{tab:orderperm}
\begin{tabular}{clccc}
\toprule
predicted & ordering & $\phi_{\text{first}}-\phi_{\text{last}}$ & measured & seed \\
rank & (three stages, final checkpoint) & (frozen) & $\mathrm{TOTAL}$ & spread \\
\midrule
$1$ & algebra $\to$ geometry $\to$ combinatorics & $+0.0808$ & $\mathbf{0.7284}$ & $0.0834$ \\
$2$ & geometry $\to$ algebra $\to$ combinatorics & $+0.0493$ & $0.6555$ & $0.0313$ \\
$3$ & algebra $\to$ combinatorics $\to$ geometry & $+0.0315$ & $0.6917$ & $0.0438$ \\
$4$ & geometry $\to$ combinatorics $\to$ algebra & $-0.0315$ & $0.6889$ & $0.0395$ \\
$5$ & combinatorics $\to$ algebra $\to$ geometry & $-0.0493$ & $0.6681$ & $0.0667$ \\
$6$ & combinatorics $\to$ geometry $\to$ algebra & $-0.0808$ & $\mathbf{0.6069}$ & $0.0437$ \\
\midrule
\multicolumn{2}{l}{Spearman $\rho$ (P1 threshold $0.60$)}
  & \multicolumn{3}{c}{$\mathbf{0.657}$ \quad\textbf{pass}} \\
\multicolumn{2}{l}{best $-$ worst (P2 threshold $2\,\mathrm{s.e.}$)}
  & \multicolumn{3}{c}{$0.1215 \pm 0.0331 = \mathbf{3.7\,\mathrm{s.e.}}$ \quad\textbf{pass}} \\
\bottomrule
\end{tabular}

\end{table}

\begin{figure}[t]
\centering
\begin{tikzpicture}[font=\small,>=Stealth]
\begin{scope}
 \node[anchor=south west,font=\bfseries\scriptsize] at (-0.2,3.75)
 {(a) the fitted potential \emph{is} the training order};
 \def\sc{29} 
 \draw[cGrey,thin] (0.55,0.25) -- (0.55,3.45);
 \foreach \phiv/\name in {0/algebra, -0.032/geometry, -0.072/calculus, -0.081/combinatorics}
 {\fill[cA] (0.55,{0.45+(\phiv+0.081)*\sc}) circle (2.3pt);
 \draw[cGrey,thin] (0.45,{0.45+(\phiv+0.081)*\sc}) -- (0.65,{0.45+(\phiv+0.081)*\sc});
 \node[font=\scriptsize,anchor=west] at (0.78,{0.45+(\phiv+0.081)*\sc})
 {\name\ \ $\phi=\phiv$};}
 \draw[cGrey,->,thick] (0.05,3.15) -- (0.05,0.55);
 \node[cGrey,font=\scriptsize,rotate=90,anchor=south] at (-0.25,1.85) {train in this order};
 \draw[cA!80!black,line width=0.9pt] (3.85,2.79) -- (4.05,2.79) -- (4.05,1.83) -- (3.85,1.83);
 \node[cA!85!black,font=\scriptsize,anchor=west,align=left,text width=2.4cm] at (4.12,2.31)
 {gap $=0.032$;\\ measured $\ASYM$\\ $=+0.0320$};
 \node[font=\scriptsize,anchor=north west,text width=6.1cm,align=left] at (-0.35,0.15)
 {Four numbers replace $4!=24$ candidate orders; for $n$ skills, $n-1$ measurements replace
 $n!$ trials. Prerequisite-first falls out, matching human curricula.};
\end{scope}
\begin{scope}[shift={(8.05,0)}]
 \node[anchor=south west,font=\bfseries\scriptsize] at (-0.55,3.75)
 {(b) six measured $\ASYM$ vs.\ $\phi_X-\phi_Y$};
 \def\ox{0.0}\def\oy{1.55}\def\sc{19}
 \draw[cGrey,thin] (\ox-0.95,\oy) -- (\ox+2.0,\oy);
 \draw[cGrey,thin] (\ox,\oy-1.15) -- (\ox,\oy+2.0);
 \foreach \v in {-0.04,0.04,0.08}
 {\draw[cGrey,thin] (\ox+\v*\sc,\oy-0.07) -- (\ox+\v*\sc,\oy+0.07);
 \node[cGrey,font=\scriptsize,anchor=north] at (\ox+\v*\sc,\oy-0.13) {$\v$};}
 \foreach \v in {0.04,0.08}
 {\draw[cGrey,thin] (\ox-0.07,\oy+\v*\sc) -- (\ox+0.07,\oy+\v*\sc);
 \node[cGrey,font=\scriptsize,anchor=east] at (\ox-0.09,\oy+\v*\sc) {$\v$};}
 \draw[cGrey,thin] (\ox-0.07,\oy-0.04*\sc) -- (\ox+0.07,\oy-0.04*\sc);
 \node[cGrey,font=\scriptsize,anchor=west] at (\ox+0.09,\oy-0.04*\sc) {$-0.04$};
 \node[cGrey,font=\scriptsize,anchor=north] at (\ox+1.05,\oy-0.95) {$\phi_X-\phi_Y$ (predicted)};
 \node[cGrey,font=\scriptsize,rotate=90,anchor=south] at (\ox-1.32,\oy+0.45) {$\ASYM$ (measured)};
 \draw[cGrey,densely dashed] (\ox-0.055*\sc,\oy-0.055*\sc) -- (\ox+0.095*\sc,\oy+0.095*\sc);
 \node[cGrey,font=\scriptsize,anchor=north west] at (\ox+0.088*\sc,\oy+0.083*\sc) {$y=x$};
 \foreach \x/\y in {0.0719/0.0743, 0.0808/0.0779, 0.0493/0.0506, 0.0315/0.0320,
 0.0089/0.0104, -0.0404/-0.0396}
 \fill[cA] (\ox+\x*\sc,\oy+\y*\sc) circle (2.2pt);
 \node[font=\scriptsize,anchor=north west,text width=5.6cm,align=left] at (-1.0,0.15)
 {r.m.s.\ residual $0.0018$; pre-registered triangle additivity $0.073$ against a
 threshold of $1.0$. The residual beats the marginal noise floor because same-seed arms
 share stage-1 checkpoints (\S\ref{sec:budget}).};
\end{scope}
\end{tikzpicture}
\caption{\textbf{At one stated budget, order collapses onto one number per skill.} (a) The
potential fitted to six volume-matched pairwise measurements; training in decreasing $\phi$ is
prerequisite-first. (b) Each pair's measured directional effect against the potential difference
the fit predicts. Six measurements are summarised by three free numbers and the pre-registered
additivity criterion passes by more than an order of magnitude, but three parameters against six
observations is weak evidence, and the sharpest of the three reasons to hold this lightly is
that four of these six pairs reverse sign at the budget that teaches.}
\label{fig:potential}
\end{figure}
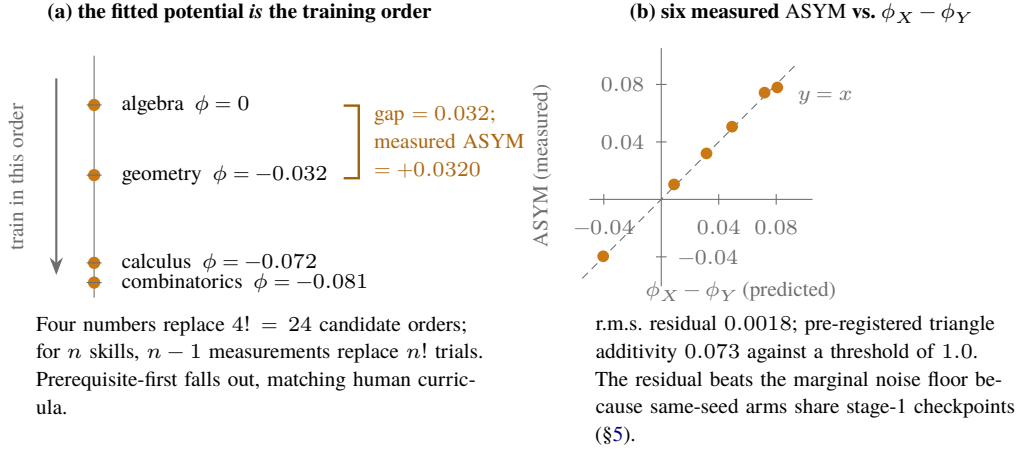

\section{Census of Every Quantitative Read-Out}\label{app:census}

The multiplicity paragraph of \S\ref{sec:license} corrects over the total number of
quantitative read-outs this paper reports. \emph{Counting rule:} one read-out per contrast,
gate, or test statistic for which the paper states a value against a floor, a $\sigma$, or a
pass/fail criterion, including read-outs reported only in summarised form (the capacity
grids), and excluding guards used solely as admission gates and the pre-registered machine
checks of Appendix~\ref{app:machine}, which test the mathematics rather than the models.

\begin{table}[ht]
\centering\footnotesize
\caption{\textbf{All $94$ read-outs, grouped.} At $\alpha=0.05$, Bonferroni over $94$ puts the
two-sided threshold at $z\approx3.5$; every firing conflict row with three seeds clears it.}
\label{tab:appcensus}
\begin{tabular}{lrl}
\toprule
family & count & where \\
\midrule
reliability $F$-tests, one per observable & $6$ & Table~\ref{tab:reliability} \\
fixed-volume contrasts, published and matched & $6$ & Table~\ref{tab:matchedthree} \\
volume doubling, four skills and their mean & $5$ & Table~\ref{tab:volume} \\
directional order, six pairs at two budgets & $12$ & Table~\ref{tab:budgetflip} \\
in-run sign change, one per trajectory pair & $4$ & \S\ref{sec:budget} \\
replication contrasts (endpoint re-runs) & $3$ & \S\ref{sec:budget} \\
switch protocol replication, two independent executions & $2$ & \S\ref{sec:budget} \\
additivity tests, matched and unmatched & $2$ & \S\ref{sec:budget} \\
conflict-program rows ($13$ with a $D$, one excluded by its guard) & $14$ & Table~\ref{tab:conflict} \\
conflict-minus-control differences & $2$ & \S\ref{sec:conflict} \\
allocation totals, three conflict corpora & $3$ & \S\ref{sec:conflict} \\
block-length sweeps, two corpora & $2$ & \S\ref{sec:conflict} \\
capacity cells, published and volume-matched & $24$ & \S\ref{sec:capacity} \\
crossing-location ratio & $1$ & \S\ref{sec:capacity} \\
scalar sign-axes, matched cross-family $|\Delta V|$ & $4$ & \S\ref{sec:capacity} \\
premise $c\rho<\gamma$ (36 pairs $\times$ 4 depths, summarised) & $1$ & Appendix~\ref{app:premise} \\
$T^\star$ ranking tests, in-model and on-model & $2$ & \S\ref{sec:license} \\
learning-rate slope of $|D|$ & $1$ & \S\ref{sec:conclusion} \\
\midrule
total & $94$ & \\
\bottomrule
\end{tabular}
\end{table}

\end{document}